\documentclass{article}

\usepackage{stfloats}
\usepackage{microtype}
\usepackage{graphicx}
\usepackage{subfigure}
\usepackage{booktabs} %
\usepackage{amsfonts}
\usepackage{amsthm}
\usepackage{xcolor}
\usepackage{algorithm}
\usepackage[noend]{algorithmic}
\usepackage{mathtools}
\usepackage{thm-restate}
\usepackage{tabularx}
\usepackage{rotating}
\mathtoolsset{showonlyrefs=true}

  {\begin{enumerate}%
    \setlength{\itemsep}{2pt}%
    \setlength{\parskip}{0pt}}%
  {\end{enumerate}}

\newcolumntype{b}{>{\hsize=1.3\hsize}X}
\newcolumntype{s}{>{\hsize=.2\hsize}X}
\newcolumntype{m}{>{\hsize=.35\hsize}X}
\newcolumntype{l}{>{\hsize=.5\hsize}X}
\newcolumntype{x}{>{\hsize=.45\hsize}X}
\newcolumntype{g}{>{\hsize=0.7\hsize}X}

\usepackage{hyperref}

\newcommand{\dataset}{\mathcal{D}}
\newcommand{\bellman}{\mathcal{B}}
\newcommand{\states}{\mathcal{S}}
\newcommand{\actions}{\mathcal{A}}
\newcommand{\E}{\mathop{\mathbb{E}}}
\newcommand{\F}{\mathcal{F}}

\newcommand{\Lphi}{\mathcal{L}_\Phi}
\newcommand{\Lpsi}{\mathcal{L}_\Psi}
\newcommand{\hatLpsi}{\hat{\mathcal{L}}_\Psi}
\newcommand{\hatLphi}{\hat{\mathcal{L}}_\Phi}
\newcommand{\ploff}{\textsc{PLOff}}

\newtheorem{definition}{Definition}
\usepackage{amsmath}

\usepackage[accepted]{icml2021}

\icmltitlerunning{Offline Reinforcement Learning with Pseudometric Learning}

\begin{document}

\twocolumn[
\icmltitle{Offline Reinforcement Learning with Pseudometric Learning}

\begin{icmlauthorlist}
\icmlauthor{Robert Dadashi}{brain}
\icmlauthor{Shideh Rezaeifar}{geneva}
\icmlauthor{Nino Vieillard}{brain,lorraine}
\icmlauthor{L\'eonard Hussenot}{brain,lille}
\icmlauthor{Olivier Pietquin}{brain}
\icmlauthor{Matthieu Geist}{brain}
\end{icmlauthorlist}

\icmlaffiliation{brain}{Google Research, Brain Team}
\icmlaffiliation{geneva}{University of Geneva}
\icmlaffiliation{lorraine}{Universit\'e de Lorraine, CNRS, Inria, IECL, F-54000 Nancy, France}
\icmlaffiliation{lille}{Univ. de Lille, CNRS, Inria Scool, UMR 9189 CRIStAL}

\icmlcorrespondingauthor{Robert Dadashi}{dadashi@google.com}

\icmlkeywords{Machine Learning, ICML}

\vskip 0.3in
]

\printAffiliationsAndNotice{} 

\begin{abstract}
\looseness=-1
Offline Reinforcement Learning methods seek to learn a policy from logged transitions of an environment, without any interaction. In the presence of function approximation, and under the assumption of limited coverage of the state-action space of the environment, it is necessary to enforce the policy to visit state-action pairs \emph{close} to the support of logged transitions. In this work, we propose an iterative procedure to learn a pseudometric (closely related to bisimulation metrics) from logged transitions, and use it to define this notion of closeness. We show its convergence and extend it to the function approximation setting. We then use this pseudometric to define a new lookup based bonus in an actor-critic algorithm: \ploff{}. This bonus encourages the actor to stay close, in terms of the defined pseudometric, to the support of logged transitions. Finally, we evaluate the method on hand manipulation and locomotion tasks.
\end{abstract}

\section{Introduction}
\label{sec:intro}
Reinforcement Learning (RL) has proven its ability to solve complex problems in recent years \citep{go,starcraft}. Behind those breakthroughs, the adoption of RL in real-world systems remains challenging \citep{rlchallenges}. Learning a policy by trial-and-error, while operating on the system,  can be detrimental to the system where it is deployed (\textit{e.g.}, user satisfaction in recommendation systems or material damage in robotics) and is not guaranteed to lead to good performance (sparse rewards problems). 

Nevertheless, in the setting where experiences were previously collected in an environment (\emph{logged transitions}), one possible way of learning a policy is to mimic the policy that generated these experiences \citep{bc}. However, if these experiences come from different sources, with different degrees of desirability, naive imitation might lead to poor results. On the other hand, offline RL (or batch RL)~\citep{lagoudakis2003least,ernst2005tree,riedmiller2005neural,levine2020offline}, offers a setting where the policy is learned from the collected experiences. As it requires no interaction with the environment, offline RL is a promising direction for learning policies that can be deployed into systems. 

Still,  collected experiences typically only cover a subset of the range of possibilities in an environment (not every state is present; for a given state, not  every action is taken). With function approximation, this is particularly problematic since actions not executed in the system can be assigned overly optimistic values (especially through bootstrapping), which leads to poor policies. To limit this extrapolation error, offline RL is typically  incentivized to learn policies that are plausible in light of the collected experiences. In other words, offline RL methods need to learn policies that maximize their return, while making sure to remain \emph{close} to the support of logged transitions. 

\looseness=-1
This work introduces a new method for computing the closeness to the support by learning a pseudometric from collected experiences. This pseudometric, close to bisimulation metrics \citep{ferns2004metrics}, computes similarity between state-action pairs based on the expected difference in rewards when following specific sequences of actions. 
We show theoretical properties in the dynamic programming setting for deterministic environments as well as for the sampled setting. We further extend the learning of this pseudometric to the function approximation setting, and propose an architecture to learn it from collected experience. We define a new offline RL actor-critic  algorithm: \ploff{} (\textbf{P}seudometric \textbf{L}earning \textbf{Off}line RL), which computes a bonus through this learned pseudometric and uses it to filter admissible actions in the greedy step and penalizes non-admissible actions in the evaluation step. Finally, we lead an empirical study  on the hand manipulation and locomotion tasks of the D4RL benchmark from \citet{fu2020d4rl}.

We make the following contributions: 1) we extend neural bisimulation metrics \citep{castro2020scalable} to state-action spaces and to the offline RL setting and 2) we exploit this pseudometric to tackle the out-of-distribution extrapolation error of offline RL by adding a simple lookup bonus to a standard actor-critic algorithm and show that it compares favorably to state-of-the art offline RL methods.
\section{Background}
\label{sec:background}

\paragraph{Reinforcement Learning.} We consider the classic RL setting~\citep{suttonbarto}, formalized with Markov Decision Processes (MDPs). An MDP is a tuple $\mathcal{M} := (\states, \actions, r, P, \gamma)$, with $\states$ the state space, $\actions$ the action space, $r: \states \times \actions \mapsto \mathbb{R}$ the expected reward function, $P : \states \times \actions \mapsto \mathcal{P}(\states)$ the transition function which maps state-action pairs to distributions over the set of states $\mathcal{P}(\states)$ and $\gamma$ the discount factor for which we assume $\gamma \in [0, 1)$. A stationary deterministic policy $\pi$ is a mapping from states to actions (the following can easily be extended to stochastic policies). The value function $V^\pi$ of a policy $\pi$ is defined as the expected discounted cumulative reward from starting in a particular state and acting according to $\pi$:
$V^\pi(s) = \E \big(\sum^{\infty}_{i=0}  \gamma^i r(s_i, \pi(s_i)) | s_0 = s)\big)$. 
The action-value function $Q^\pi$ is defined as the expected cumulative reward from starting in a particular state, taking an action and then acting according to $\pi$:
$Q^\pi(s, a) = r(s, a) + \gamma \E \big(V^\pi(s')\big)$. 
The~\citet{bellman} operator $\bellman$  connects an action-value function $Q$ for the state-action pair $(s,a)$ to the action-value function in the subsequent state $s'$: 
$\bellman^\pi (Q)(s, a) := r(s, a) + \gamma \E\big( Q(s', \pi(s'))\big)$. 
$Q^\pi$ is the (unique) fixed-point of this operator, and the difference between $Q(s,a)$ and its image through the Bellman operator $\|Q - \bellman^\pi Q \|$ is called a temporal difference error.\\

An optimal policy $\pi^*$ maximizes the value function $V^{\pi^*}$ for all states. In continuous state-action spaces, actor-critic methods \citep{konda} are a common paradigm to learn a near-optimal policy. In this work we only consider deterministic policies; we justify this restriction by the fact that stochastic policies are desirable because of their side effect of exploration \citep{pmlr-v80-haarnoja18b}, but in this case we want to learn a policy with near-optimal behavior without interaction with the environment. Therefore, we use the actor-critic framework of \citet{silver2014deterministic}. It consists in concurrently learning a parametrized policy $\pi_\theta$ and its associated parametrized action-value function $Q_\omega$. $Q_\omega$ minimizes a temporal difference error $\|Q_\omega(s, a) - r(s,a) - Q_{\bar{\omega}}(s', \pi_\theta(s')) \|$ (with $Q_{\bar{\omega}}$ a target action-value function, tracking $Q_w$), and $\pi_\theta$ maximizes the action-value function $Q_\omega(s, \pi_\theta(s))$.\\

In the classical RL setting, transitions $(s, a, s', r)$ are sampled through interactions with the environment. In on-policy actor-critic methods \citep{Sutton1999PolicyGM,Schulman2015TrustRP}, updates on $\pi_\theta$ and $Q_\omega$ are made as the policy gathers transitions by interacting in the environment. In off-policy actor-critic methods \citep{lillicrap2019continuous,pmlr-v80-haarnoja18b,fujimoto2019off}, the transitions gathered by the policy are stored in a replay buffer and sampled using different sampling strategies \citep{experiencereplay,schaul2015prioritized}. These off-policy methods extend to the offline RL setting quite naturally. The difference is that in the offline RL setting, transitions are not sampled through interactions from the environment, but from a fixed dataset of transitions. 

\looseness=-1
Throughout the paper $\dataset = \{(s_i, a_i, r_i, s'_i)\}_{1:N}$ is the dataset of $N$ transitions  collected in the considered environment. To ease notations we  write $s \sim \dataset$, $s, a \sim \dataset$, $r \sim \dataset$ to indicate that a transition $(s, a, s', r)$ is sampled at random from this dataset, and that we only consider the associated state $s$, state-action pair $(s, a)$ or reward $r$ respectively.

\paragraph{Pseudometric in MDPs.} A core issue in RL is to define a meaningful metric between states or state-action pairs \citep{Lan2021MetricsAC}. Consider for example a maze, two states could be close according to the Euclidean distance, but far away in terms of the minimal distance an agent would have to travel to join one state from the other (due to walls). In this case, a relevant metric is the distance in the graph formed from state transitions induced by the MDP \citep{protovaluefunctions}. We consider the following relaxed notion of metric\footnote{A metric is a pseudometric for which $d(x, y) = 0 \Rightarrow x = y.$}:
\begin{definition}[Pseudometric]
\label{def:pseudometric}
Given a set $M$, a pseudometric is a function $d : M \times M \mapsto \mathbb{R}_+$ such that, $\forall x, y, z \in M$, we have $d(x; x) = 0$, $d(x; y) = d(y; x)$,   $d(x; z) \leq d(x; y) + d(y; z)$.
\end{definition}

In the context of Markov Decision Processes, bisimulation relations \citep{givan2003equivalence} are a form of state abstraction \citep{li2006towards}, based on an equivalence relation. They are defined by the following recurrent definition: two states are bisimilar if they yield the same expected reward and transition to bisimulation equivalence classes with equal probability. This definition is too restrictive to be useful in practice. \citet{ferns2004metrics} introduce bisimulation metrics which are pseudometrics that soften the concept of bisimulation relations. Bisimulation metrics are defined on the state space $\states$. Denote $\mathbb{M}_\states$ the set of pseudometrics  on $\states$. The bisimulation metric is the (unique) fixed point of the operator $\mathcal{F_\states}$ defined as:
\begin{align}
\F_\states & (d)(s; t) := \\
&\max_{a \in \actions} \Big(|r(s, a) - r(t, a)| + \gamma \mathcal{W}_1(d)(P(s, a),P(t, a))\Big)
\end{align}
where $\mathcal{W}_1(d)$ is the 1-Wasserstein distance \citep{villani} with the distance between states measured according to the pseudometric $d$. Therefore, the bisimulation metric is the limit of the repeated application of the sequence $\big(\mathcal{F_\states}^n(d_0)\big)_{n \in \mathbb{N}}$ for any initial $d_0 \in \mathbb{M_\states}$.

Although pseudometrics in MDPs have proven to be effective in some applications \citep{melo,gmmil,pwil}, they are usually hand-crafted or learned with ad-hoc strategies.

\section{Method\label{sec:method}}
We present the overall idea of our method in Sec. \ref{sec:lookup_bonus}: an offline RL algorithm which is incentivized to remain close to the support of collected experiences using a pseudometric-based bonus. In Sec. \ref{sec:pseudometric_learning} to \ref{sec:approx}, we present how to learn this pseudometric, from a theoretical motivation to a gradient-based method. We give practical considerations to derive the bonus in Sec. \ref{sec:tractable} and provide the resulting algorithm: \ploff{}, in Sec. \ref{sec:algorithm}.

\subsection{Offline RL with lookup bonus}
\label{sec:lookup_bonus}
For the time being, let us assume  the existence of a pseudometric $d$ on the state-action space  $\states \times \actions$. We can infer a distance $d_\dataset$ from a transition to the dataset $\dataset$: 
\begin{align}
    d_\dataset(s, a) = \min_{\hat{s}, \hat{a} \in \dataset} d(s, a; \hat{s}, \hat{a}).
\end{align}

This distance to the dataset $\dataset$, also referred to as the projection distance, is simply the distance from $(s, a)$ to the nearest element of $\dataset$. It is central to our work since it defines the notion of closeness to the support of transitions $\dataset$. From $d_\dataset$, we can infer a bonus $b$ using a monotonically decreasing function $f : \mathbb{R} \mapsto \mathbb{R}$: 
\begin{align}
    b(s, a) = f(d_\dataset(s, a)).
\end{align}

Note that the concept of a bonus is overloaded in RL; it typically applies to exploration strategies \citep{jurgen,thrun,Bellemare2016UnifyingCE}. In our case, the bonus $b$ is opposite to exploration-based bonuses since it will encourage the policy to act similarly to existing transitions of the collected experiences $\dataset$. In other words, it prevents exploring too far from the dataset. 

We adapt the actor-critic framework by adding the bonus to the actor maximization step and the critic minimization step (with difference multipliers $\alpha_a$ and $\alpha_c$). We learn a parametrized policy $\pi_\theta$ and its corresponding action-value function $Q_\omega$. In a schematic way, we sample transitions $(s, a, r, s') \in \dataset$ and minimize the two following losses:
\begin{align}
    \text{(critic)} \;\; \min_\omega \|&Q_\omega(s, a) - r(s,a) \\ &- \gamma Q_{\bar{\omega}}(s', \pi_\theta(s')) - {\color{blue}\alpha_c b(s', \pi_\theta(s'))}\|, \label{eq:critic_step}
\end{align}
\begin{align}
\text{(actor)} \;\; \max_\theta Q_\omega(s, \pi_\theta(s)) + {\color{blue} \alpha_a b(s, \pi_\theta(s))} \label{eq:actor_step}.
\end{align}

This modification of the actor-critic framework is common in offline RL \citep{buckman2020importance}. The bonus $b$ typically consists in a measure of similarity between an estimated behavior policy that generated $\dataset$ and the policy $\pi$ we learn \citep{pmlr-v80-fujimoto18a,wu2019behavior,kumar2019stabilizing}.

\subsection{Pseudometric Learning}
\label{sec:pseudometric_learning}
To define a bonus $b$, we first learn a pseudometric $d$ on the state-action space $\states \times \actions$ similarly to bisimulation metrics \citep{ferns2004metrics,ferns2011bisimulation,ferns2012methods}, with the difference being that we are interested in pseudometrics in state-action space $\states \times \actions$ rather than state space $\states$.
We will show that the pseudometric $d$ we are interested in is the fixed point of an operator $\F$. In the following, we assume that the MDP is deterministic.

Let $\mathbb{M}$ be the set of bounded pseudometrics on $\states \times \actions$. We define the operator $\F: \mathbb{M} \mapsto \mathbb{M}$ as follows: for two state-action pairs $(s_1, a_1)$ and $(s_2, a_2)$, that maps to next states $s'_1$ and $s'_2$ we have:
\begin{align}
    \F(d)&(s_1,a_1;s_2,a_2) := \\
    &|r(s_1, a_1) - r(s_2, a_2)| + {\gamma\E}_{a' \sim \mathcal{U}(\actions)} d(s'_1,a';s'_2,a'),
\end{align}
with $\mathcal{U}(\actions)$ the uniform distribution over actions.

This operator is of particular interest as it takes a distance $d$ over state-action pairs as input, and outputs a distance $\F(d)$ which is the distance between immediate rewards, plus the discounted expected distance between the two transitioning states for a random action. Notice that contrary to bisimulation metrics, we do not use a maximum over next actions for multiple reasons: the use of a maximum can be overly pessimistic when computing the similarity \citep{castro2020scalable}; in the case of continuous action spaces a maximum is hard to estimate; and finally in the presence of function approximation (Section \ref{sec:approx}) it can lead to instabilities. 

We now establish a series of properties of the operator $\mathcal{F}$, all being proven in Appendix \ref{sec:proofs}.

\begin{restatable}{proposition}{stability}
Let $d$ be a pseudometric in $\mathbb{M}$, then $\F(d)$ is a pseudometric in $\mathbb{M}$.
\end{restatable}

This result indicates the stability of a pseudometric through the operator $\F$ which is not trivial. The stability comes from the deterministic nature of the MDP as well as the fact that the boostrapped estimate is computed for the \emph{same} action at the next states. We are now interested in the repeated application of this operator $\big(\mathcal{F}^n(d_0)\big)_{n \in \mathbb{N}}$ starting from the $0$-pseudometric $d_0$ (mapping all pairs to 0).\\

\begin{restatable}{proposition}{contraction}
Let $d$ be a pseudometric in $\mathbb{M}$. We note $\|d\|_\infty$ as $\max_{s, s' \in \states} \max_{a, a' \in \actions} d(s, a; s', a')$. The operator $\F$ is a $\gamma$-contraction for $\|\cdot \|_\infty$.
\end{restatable}

Since the operator $\mathcal{F}$ is a contraction, it follows that the sequence $\big(\mathcal{F}^n(d_0)\big)_{n \in \mathbb{N}}$ converges to the pseudometric of interest $d^*$ (it would for any initial pseudometric, but $d_0$ is of particular empirical interest).\\

\begin{restatable}{proposition}{fixedpoint}
$\F$ has a unique fixed point $d^*$ in $\mathbb{M}$. Suppose $d_0 \in \mathbb{M}$ then $\lim_{n \to \infty} \F^n(d_0) = d^*$.
\end{restatable}

The fixed point of $\F$ can be thought of as the similarity between state-actions, measured by the difference in immediate rewards added to the difference in rewards in future states if the sequence of actions is selected uniformly at random.
In other words, two state-action pairs will be close if 1) they yield the same immediate reward and 2) following a random walk from the resulting transiting states yields a similar return.

\subsection{Pseudometric learning with sampling}
\label{sec:sampling}

Now, we move to the more realistic setting where the rewards and dynamics of the MDP are not known. Thus, the MDP is available through sampled transitions. We assume in this section that we have a finite state-action space. We define an operator $\hat{\F}$, which is a sampled version of $\F$: suppose we sample a pair of transitions from the environment $(\hat{s}_1, \hat{a}_1, \hat{s}'_1, \hat{r}_1), (\hat{s}_2, \hat{a}_2, \hat{s}'_2, \hat{r}_2)$, we have:
\begin{align}
  \hat{F}&(d)(s_1,a_1;s_2,a_2) = \\
  &\left\{\begin{aligned}
    &\left|\hat{r}_1-\hat{r}_2 \right| + \begin{aligned}[t] &{\gamma\E}_{u' \sim \mathcal{U}(\actions)}\, d(\hat{s}'_1,u',\hat{s}'_2, u') \\
    &\text{if } s_1, a_1, s_2, a_2 = \hat{s}_1, \hat{a}_1, \hat{s}_2, \hat{a}_2,
        \end{aligned}\\
    &d(s_1,a_1;s_2,a_2)\;\; \text{otherwise.}
    \end{aligned}\right.
\end{align}
Similarly to what is observed in the context of bisimulation metrics by \citet{castro2020scalable}, the sampled version $\hat{F}$ has similar convergence properties as $\F$.\\

\begin{restatable}{proposition}{sampleconvergence}
Suppose sufficient coverage of the state-action space: $\exists \epsilon > 0$ such that for any pairs of state-action pairs $(s, a), (\hat{s}, \hat{a}) \in (\states \times \actions) \times (\states \times \actions)$, $(s, a), (\hat{s}, \hat{a})$ is sampled with at least probability $\epsilon$, then the repeated application of $\hat{\F}$ converges to the fixed point $d^*$ of $\F$.
\end{restatable}

If the environment is stochastic, the repeated application of $\hat{\F}$ does not converge to the fixed point of the operator $\F$. In fact, it is not even stable in the space of pseudometrics (the expectation and absolute value are not commutative). This results appears as a limitation of our work, since it only applies to deterministic environments. We leave to future work whether we can define a different pseudometric (fixed point of another operator) that would have convergence guarantees in the sampling case. As we only evaluate our approach on deterministic environments (Section \ref{sec:experiments}), another interesting direction is whether our approach empirically extends to stochastic environments (albeit less principled).

\subsection{Pseudometric learning with approximation}
\label{sec:approx}

Building upon the insights of the previous sections, we  derive an approximate version of the iterative scheme, to estimate a near-optimal pseudometric $d^*$. Pairs of transitions are assumed to be sampled from a fixed dataset $\dataset$. We use Siamese neural networks $\Phi$ \citep{bromley1994signature} to derive an approximate version of a pseudometric $d_\Phi$. Using Siamese networks to learn a pseudometric \citep{castro2020scalable} is natural since it respects the actual definition of a pseudometric by design (Definition \ref{def:pseudometric}). To ease notations, we conflate the definition of the deep network $\Phi$ with its parameters. We define the pseudometric $d_\Phi$ as:
\begin{align}
    d_\Phi(s_1, a_1; s_2, a_2) = \| \Phi(s_1, a_1) - \Phi(s_2, a_2) \|,
\end{align} where $\|\cdot\|$ is the Euclidean distance.

From the fixed-point iteration scheme defined in Section \ref{sec:sampling}, we want to define a loss to retrieve the fixed point $d^*$. Similarly to fitted value-iteration methods \citep{Bertsekas1996NeuroDynamicP,Munos2008FiniteTimeBF}, which is the basis for the DQN algorithn \citep{Mnih2015HumanlevelCT}, we consider the parameters of the image of the operator $\hat{\F}$ to be fixed, and note it $\hat{\F}(d_{\bar{\Phi}})$. We thus learn $d_\Phi$ by minimizing the following loss, which is exactly the temporal difference error $\big( \hat{\F}(d_{\bar{\Phi}}) - d_\Phi \big)^2$: 
\begin{align}
    \Lphi = \E_{\substack{s_1, a_1, r_1, s'_1 \sim \dataset\\
                  s_2, a_2, r_2, s'_2 \sim \dataset}}
                  \Big( 
    \begin{aligned}[t]
                  d_\Phi(&s_1, a_1; s_2, a_2) - |r_1 - r_2| - \\
    & \gamma \E_{a' \in \mathcal{U}(\actions)} d_{\bar{\Phi}}(s'_1,a';s'_2,a')\Big)^2.
    \end{aligned}
\end{align}
We introduce another pair of Siamese networks $\Psi$ (again we conflate the definition of the network with its parameters) to track the bootstrapped estimate $\E_{a' \in \mathcal{U}(\actions)} d_{\bar{\Phi}}(s_1,a',s_2,a')$, that we learn minimizing the following loss:
\begin{align}
    \Lpsi = \E_{\substack{s_1 \sim \dataset, s_2 \sim \dataset}} \Big(
 \begin{aligned}[t]
                  &\| \Psi(s_1) - \Psi(s_2) \| -  \\
                  &\E_{a' \in \mathcal{U}(\actions)} d_{\bar{\Phi}}(s_1,a'; s_2,a') \Big)^2.
 \end{aligned}
\end{align}

\begin{figure*}[t]
\centering
\includegraphics[width=0.9\linewidth]{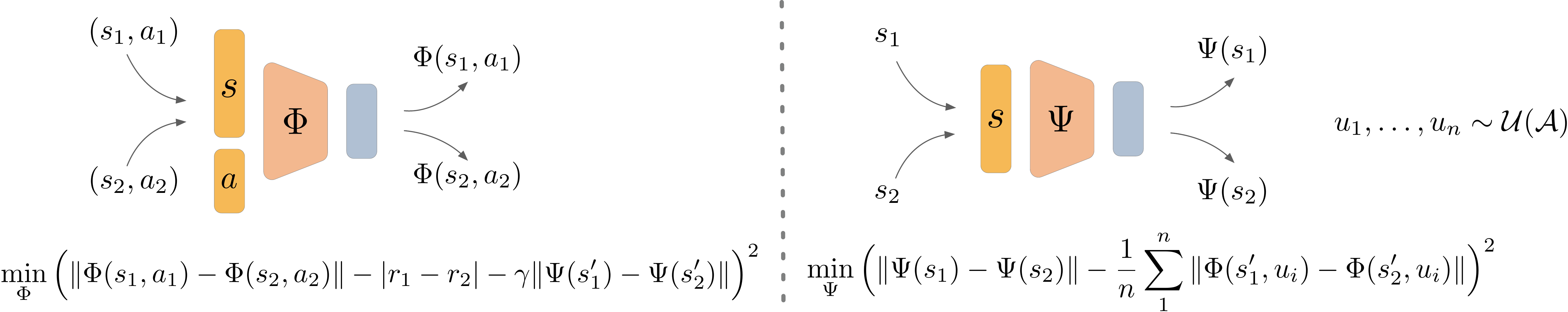}
\caption{Architecture details of pseudometric learning. Two pairs of Siamese networks $\Phi$ and $\Psi$ are concurrently optimized. Left: the pseudometric $d_\Phi$ on $\states \times \actions$. Right: The pseudometric $d_\Psi$ on $\states$ tracking the boostrapped estimate of $\F(d_\Phi)$.}
\label{fig:archi}
\end{figure*}
We justify this design choice in Section \ref{sec:tractable}, where we show that this makes the derivation of the bonus tractable. Therefore, we optimize the two following losses $\hatLpsi$ and $\hatLphi$:

\begin{align}
    \hatLphi = \E_{\substack{s_1, a_1, r_1, s'_1 \sim \dataset\\
                  s_2, a_2, r_2, s'_2 \sim \dataset}}
                  \Big(
                  \begin{aligned}[t]
                  & \|\Phi(s_1, a_1) - \Phi(s_2, a_2)\| \\
                  & - |r_1 - r_2| - \gamma |\Psi(s'_1) - \Psi(s'_2) |\Big)^2,
                  \end{aligned}
\end{align}
\begin{align}
\hatLpsi = \E_{\substack{s_1 \sim \dataset, s_2 \sim \dataset}} \Big(
 \begin{aligned}[t]
                  &\| \Psi(s_1) - \Psi(s_2) \| -  \\
                  &\frac{1}{n}\sum_{\substack{u_1, \dots, u_n \\ \sim \mathcal{U}(\actions)}} \|\Phi(s_1, u_1) - \Phi(s_2, u_2) \| \Big)^2.
 \end{aligned}
\end{align}
\vspace{-25pt}

{A visual representation of the two pairs of Siamese networks as well as their losses is provided in Figure~\ref{fig:archi}. Remark that the proposed architecture seems to present similar caveats as naive offline RL approaches (since we estimate quantities that might not be present in the dataset of collected experiences). However, here, the divergence of the quantity at hand (the pseudometric learned) is unlikely since the goal is to minimize a positive quantity. This makes the problem of learning $d_\Phi$ inherently more stable than learning an optimal policy. A limitation of this method is that it relies on the reward function $r$ to build similarities between state-action pairs, therefore in very sparse reward environments or with very limited coverage of the state-action space, the quality of the pseudometric learned might conflate state-action pairs together, and hence be less adapted to learn a meaningful measure of similarity.}

\subsection{Tractable bonus}
\label{sec:tractable} 
Once the pseudometric $d_\Phi$ is learned, we can define a lookup bonus introduced in section \ref{sec:lookup_bonus}. Given a monotonously decreasing function $f$, we have: $b(s, a) = f\big(\min_{(\hat{s}, \hat{a}) \in \dataset}d_\Phi(\hat{s}, \hat{a}; s, a) \big)$. This bonus has a complexity that is linear in the size of $\dataset$ and in the dimension of the representation $\Phi(s, a)$. As we are considering datasets with large numbers of transitions ($\sim 10^6$), this makes the exact derivation of the bonus computationally expensive.

Therefore we pre-compute the $k$-nearest neighbors of each state $s \in \dataset$ according to the Euclidean distance $d_\Psi$ induced by $\Psi$;
$d_\Psi(s_1, s_2) = \| \Psi(s_1) - \Psi(s_2) \|$. We note: 
\begin{align}
    \mathcal{H}(s) = \Big\{ (\hat{s}, \hat{a})  \in \dataset | \; \hat{s} \text{ is a $k$-nearest neighbor of $s$ for $d_\Psi$} \Big\}.
\end{align}
We infer the approximate distance bonus:
\begin{align}
    \bar{b}(s, a) = f \big(\min_{\hat{s}, \hat{a} \in \mathcal{H}(s)} d_\Phi(s, a; \hat{s}, \hat{a})\big).
\end{align}

Pre-computing the $k$-nearest neighbors is expensive (the brute force complexity is quadratic in the size of the dataset, and linear in the dimension of the representation $\Psi$). In our experiments, we use a kd-tree algorithm \citep{kd_tree} from scikit-learn \citep{sklearn}. With  multiprocessing ($\sim 50$ CPUs), pre-computing nearest neighbors did not take more than a couple hours even for the largest dataset ($2.10^6$ transitions). If the size of the dataset were to be larger, we can naturally scale our method with approximate nearest neighbor methods.

\subsection{Algorithm}
\label{sec:algorithm}
We now compile the results from this section and present the pseudocode of our method in Algorithms \ref{alg:pseudometric} and \ref{alg:act-cri}. We refer to the combination of both  as \ploff{} (\textbf{P}seudometric \textbf{L}earning \textbf{Off}line RL).

\begin{algorithm}%
    \begin{algorithmic}[1]
    \floatname{algorithm}{Procedure}
    \STATE{Initialize $\Phi$, $\Psi$ networks.} 
    \FOR {step $i = 1$ to $N$}
    \STATE Train $\Phi$: $\min_\Phi \hatLphi$
    \STATE Train $\Psi$: $\min_\Psi \hatLpsi$
    \ENDFOR
    \STATE Initialize $k$-nearest neighbors array $H$.
    \FOR {step $j = 1$ to $|\dataset|$}
    \STATE Compute k-nearest neighbors of $\Psi(s_j)$.
    \STATE Add k-nearest neighbors to the array $H$.
    \ENDFOR
    \end{algorithmic}
    \caption{Bonus learning.}
    \label{alg:pseudometric}
\end{algorithm}

\begin{algorithm}%
    \begin{algorithmic}[1]
    \floatname{algorithm}{Procedure}
    \STATE{Initialize action-value network $Q_\omega$, target network $Q_{\bar{\omega}}$, $Q_\omega$ and policy $\pi_\theta$.} 
    \FOR {step $i = 0$ to $K$}
    \STATE{Train $Q_\omega$: $\min_\omega \big(Q_\omega(s, a) - r - Q_{\bar{\omega}}(s', \pi_\theta(s')) - {\color{blue}\alpha_c \bar{b}(s', \pi(s'))}\big)^2$}
    \STATE{Train $\pi_\theta$: $\max_\theta Q_\omega(s, \pi_\theta(s)) + {\color{blue} \alpha_a \bar{b}(s,\pi(s))} $}
    \STATE{Update target network $Q_{\bar{\omega}} := Q_\omega$}
    \ENDFOR
    \end{algorithmic}
    \caption{Actor-Critic Training.}
    \label{alg:act-cri}
\end{algorithm}

\newpage
\section{Experiments \label{sec:experiments}}
In this section we conduct an experimental study for the proposed approach. We evaluate it on a series of hand manipulation tasks \citep{Rajeswaran-RSS-18}, as well as MuJoCo locomotion tasks \citep{mujoco,brockman2016openai} with multiple data collection strategies from \citet{fu2020d4rl}. We first show the details of the learning procedure of the pseudometric, before showing its performance against several baselines from \citet{fu2020d4rl}. All implementation details can be found in Appendix \ref{sec:implementation_details}.

\subsection{Evaluation environments}
We evaluate \ploff{} on four hand manipulation tasks \citep{Rajeswaran-RSS-18}: nailing a hammer, opening a door, manipulating a pen and relocating a ball. We also evaluate \ploff{} on MuJoCo locomotion tasks \citep{brockman2016openai} where the goal is to maximize the distance traveled: Walker2d, HalfCheetah and Hopper. We provide visualization of the environments in Figure \ref{fig:env}. For each environment we consider multiple datasets $\dataset$ from the D4RL benchmark \citep{fu2020d4rl}. On hand manipulation tasks, these datasets are the following, "human": transitions collected by a human operator, "cloned": transitions collected by a policy trained with behavioral cloning interacting in the environment + the initial demonstrations, "expert": transitions collected by a fine-tuned RL policy interacting in the environment. On locomotion tasks, the datasets are the following, "random": transitions collected by a random policy, "medium-replay" the first 1M transitions collected by a SAC agent \citep{pmlr-v80-haarnoja18b} trained from scratch on the environment, "medium" transitions collected by a policy with suboptimal performance, "medium-expert": transitions collected by a near optimal policy + transitions collected by a suboptimal policy. To have comparable range of rewards between environments, we scale offline rewards by scaling them in $(0, 1)$ by $r := (r - \min_\dataset r) / (\max_\dataset r - \min_\dataset r)$ and learn a policy on this scaled reward.

\begin{figure*}[t!]
\centering
\includegraphics[width=\linewidth]{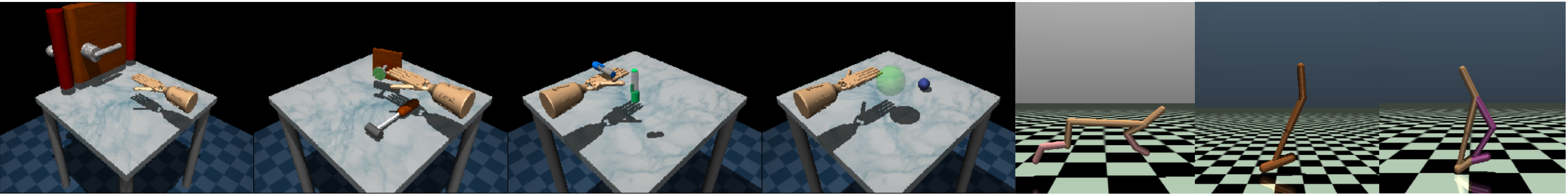}
\caption{Visualization of the environments considered. From left to right: Door, Hammer, Pen, Relocate, HalfCheetah, Hopper, Walker2d.}
\label{fig:env}
\end{figure*}

\subsection{Pseudometric learning}
We concurrently learn the deep networks $\Phi$ and $\Psi$ by minimizing the losses $\hatLphi$ and $\hatLpsi$. State-action pairs are concatenated (and states only in the case of $\Psi$) and are passed to a 2-layer network of layers sizes (1024, 32), with a relu activation on top of the first layer. Note that the concatenation step could be preceded by two disjoint layers to which the state and action are passed (thus making it more handy for visual-based obseravations).  We sample 256 actions to derive the bootstrapped estimate (loss $\hatLpsi$). We optimize $\hatLphi$ and $\hatLpsi$ using the Adam optimizer \citep{DBLP:journals/corr/KingmaB14} with batches of state-action pairs and states of size 256.

To present a qualitative intuition on the nature of the pseudometric learned, we first present the learned pseudometric on a gridworld environment with walls presented in Figure \ref{fig:ploff_gridworld}. There is a single reward state and we impose a time limit of 50 steps. We gather all transitions visited by the Q-learning algorithm trained for 500 episodes, with $\epsilon$-exploration ($\epsilon=0.1$) and a discount factor $\gamma=0.99$. States and actions are both represented using one-hot encoding. We represent the learned distances (in the sense of $d_\psi$) from the central state and from the goal state to the rest of the states. Interestingly, the distance learned takes into account the geometry of the environment, hence showing that it is task relevant. In the left part of the gridworld (far from reward), states tend to be conflated together which highlights the limitation of our work in sparse environments.

\begin{figure*}[h!]
    \centering
    \includegraphics[width=0.75\linewidth]{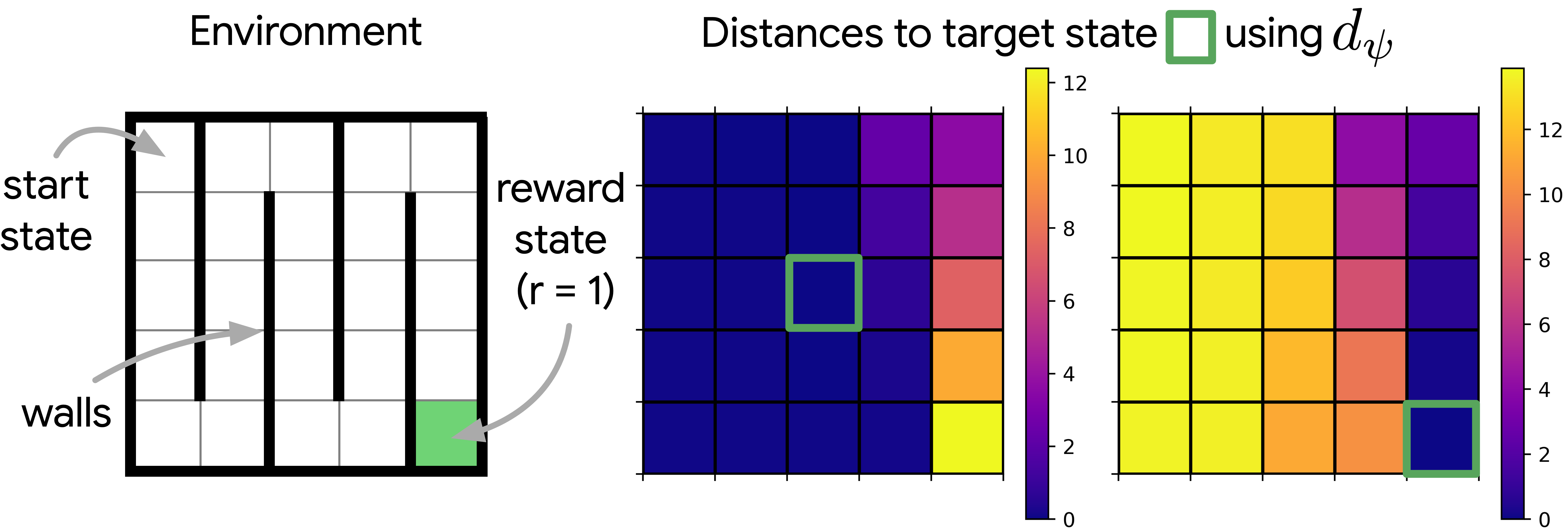}
\caption{Visualization of the learned pseudometric (center and right) in a gridworld (left). The distance from each state to the "central" state is represented in the center figure, and to the reward state in the right figure. \label{fig:ploff_gridworld}}
\end{figure*}

For the D4RL environments, we show in Figure \ref{fig:learning_curve} the decreasing learning curves for $\hatLphi$ and $\hatLpsi$. In Figure \ref{fig:noise_influence}, we show that the distribution of the learned distance $d_\Phi$ between state-action couples and perturbated versions of themselves (with Gaussian noise either on the state or the action). We show that the distance respects the intuition that the greater the perturbation is, the larger the distance becomes. Finally we provide visualizations of the state similarities learned by $\Psi$ in Appendix \ref{sec:metrics_viz}.

\begin{figure}[h!]
    \centering
    \begin{minipage}{0.48\textwidth}
        \includegraphics[width=\linewidth]{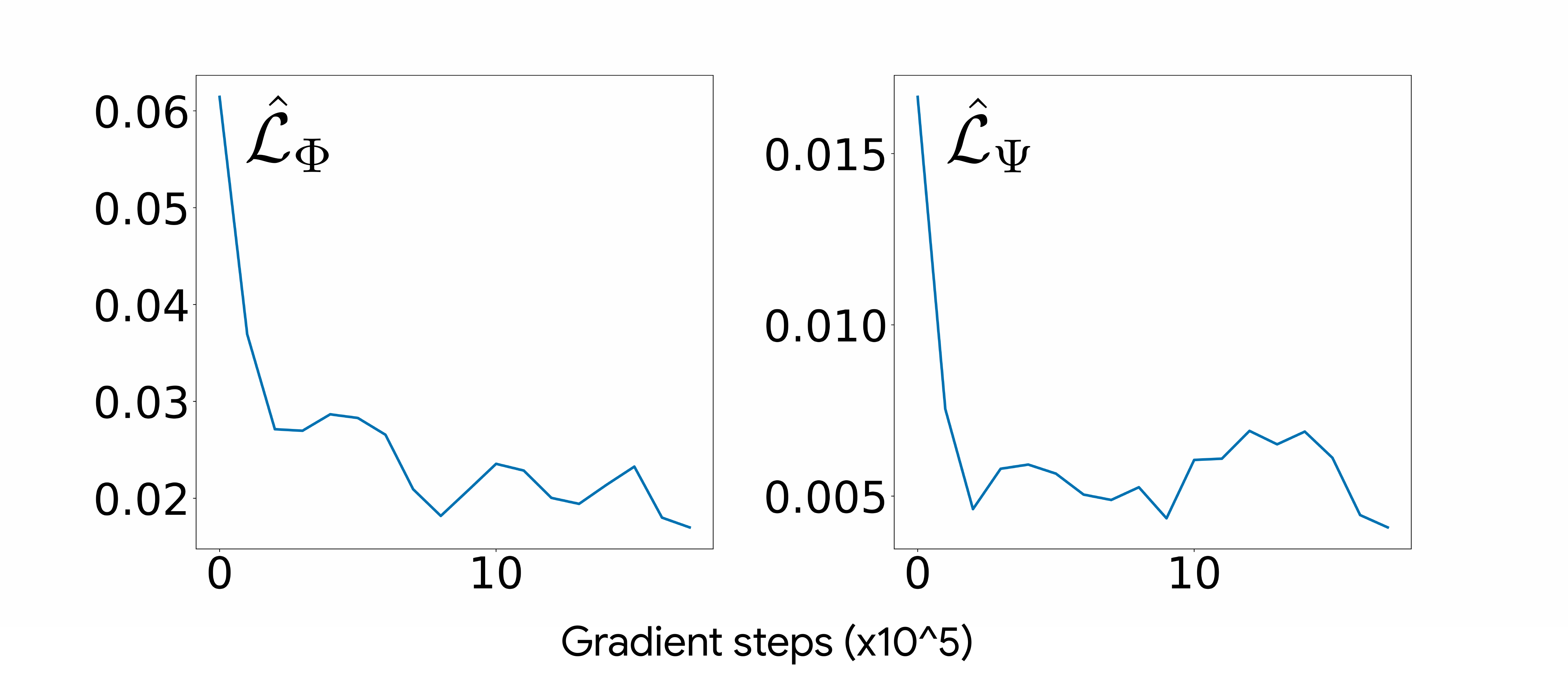}
        \caption{Learning curve of $\Phi$ and $\Psi$, for the Walker2d environment together with the "medium-replay" dataset from \citet{fu2020d4rl}. We show the values (averaged over batch) of $\hatLphi$ (left) and $\hatLpsi$ (right) throughout the learning procedure. \label{fig:learning_curve}}
    \end{minipage}\hfill
    \begin{minipage}{0.48\textwidth}
        \vspace{11pt}
        \includegraphics[width=\linewidth]{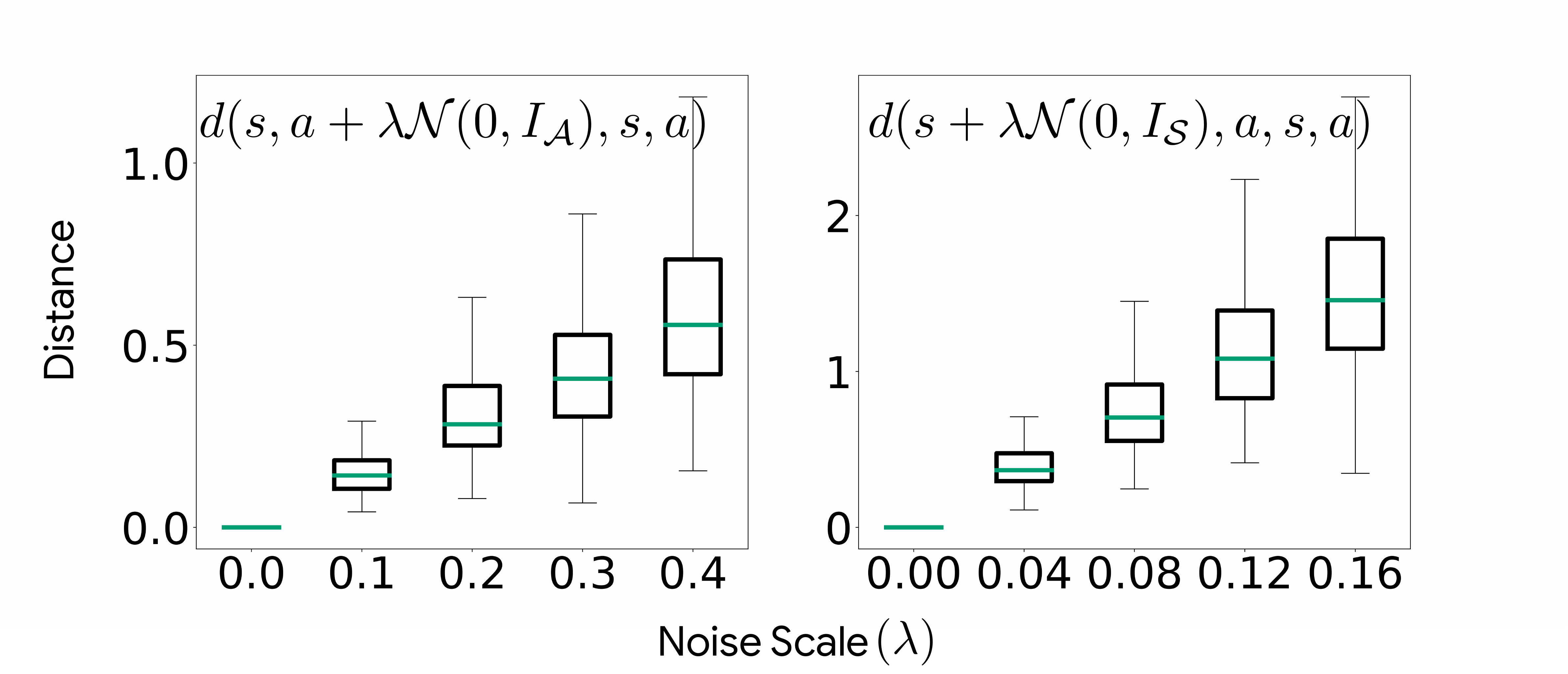}
        \caption{Influence of noise on the distance. We show on the left the learned distance of a state-action pair $(s, a)$ to a perturbated version of itself $(s, a + \lambda \mathcal{N}(0, I_\actions)$. We show on the right the learned distance of a state-action pair $(s, a)$ to a perturbated version of itself $(s+ \lambda \mathcal{N}(0, I_\states), a)$. \label{fig:noise_influence}} %
    \end{minipage}
\end{figure}

\subsection{Agent training with pseudometric bonus}
In this section, we empirically evaluate the performance of an agent trained with a bonus based on the pseudometric. In our experiments, we use TD3 \citep{fujimoto2019off}. It is an off-policy actor-critic algorithm \citep{konda} which builds on top of DDPG \citep{lillicrap2019continuous}. We use an implementation of TD3 where we load the dataset of experiences $\dataset$ in its replay buffer. We concurrently learn a policy $\pi_\theta$ and an action-value function $Q_\omega$. The critic loss and actor loss are modified to incorporate the pseudometric bonus, as described in Algorithm \ref{alg:act-cri}.

We found that the bonus defined as $\bar{b}(s, a) :=  Q_{\bar{\omega}}(s, a) \exp(-\beta d_\dataset(s, a))$ leads to strong empirical results. This bonus form is quite natural, since it uses the action-value function to scale the value of the bonus (and therefore enables hyperparameters to be more robust).

We ran a hyperparameter search on $\alpha_a, \alpha_c \in \{1, 5, 10\}$ and $\beta \in \{0.1, 0.25, 0.5\}$. We select the best hyperparameters for each family of environments (locomotion and hand manipulation), and re-run for 10 seeds. For each seed we evaluate the resulting policy on 10 episodes.
We report performance in Table \ref{tab:performance}. We compare the method with numerous baselines: AWR \citep{peng2019advantage}, Behavioral Cloning \citep{bc}, BEAR \citep{kumar2019stabilizing}, BRAC \citep{wu2019behavior}, BCQ \citep{pmlr-v80-fujimoto18a} and CQL \citep{kumar2020conservative}.

\begin{table*}[h!]
  \centering
\begin{tabularx}{\linewidth}{gssmssslll}
\toprule
Algorithm &                        BC &  BEAR & BRAC-v &   AWR &    BCQ &    CQL &  \ploff{} & \ploff-L2 & \textsc{TD3-Off} \\
\midrule
cheetah-rand        &    2.1 &  25.1 &  31.2 &   2.5 &    2.2 &   \textbf{35.4} &   1.7 $\pm$   0.1 &   2.2 $\pm$   0.0 &   28.0 $\pm$ 2.6 \\
walker-rand           &    1.6 &   \textbf{7.3} &  1.9 &   1.5 &    4.9 &    7.0 &    0.7 $\pm$   1.1 &   4.7 $\pm$   7.3 &   1.1 $\pm$ 1.2 \\
hopper-rand             &    9.8 &  11.4 &  \textbf{12.2} &  10.2 &   10.6 &   10.8 &   11.5 $\pm$   0.1 &   11.6 $\pm$   0.2 &   0.9 $\pm$   0.4 \\
cheetah-med        &   36.1 &  41.7 &   \textbf{46.3} &  37.4 &   40.7 &   44.4 &   38.2 $\pm$   0.4 &   38.0 $\pm$   0.4 &   0.3 $\pm$   4.4 \\
walker-med           &    6.6 &  59.1 &   \textbf{81.1} &  17.4 &   53.1 &   79.2 &   73.2 $\pm$   7.4 &   72.2 $\pm$   3.4 &   0.0 $\pm$   0.3 \\
hopper-med             &   29.0 &  52.1 &   31.1 &  35.9 &   54.5 &  58.0 &   \textbf{87.0 $\pm$   18.2} &   76.7 $\pm$   10.3 &   0.9 $\pm$   0.5 \\
cheetah-med-rep &   38.4 &  38.6 &  \textbf{47.7} &  40.3 &   38.2 &   46.2 &   38.5 $\pm$   1.6 &   39.0 $\pm$   1.3 &   35.9 $\pm$   3.6 \\
walker-med-rep    &   11.3 &  19.2 &  0.9 &  15.5 &   15.0 &   \textbf{26.7} &   \underline{23.2 $\pm$   9.3} &   14.6 $\pm$   6.9 &   8.0 $\pm$   4.3 \\
hopper-med-rep      &   11.8 &  33.7 &  0.6 &  28.4 &   33.1 &   \textbf{48.6} &   33.5 $\pm$   10.1 &   \underline{48.2 $\pm$   11.5} &   9.9 $\pm$   9.4 \\
cheetah-med-exp &   35.8 &  53.4 &     41.9 &  52.7 &   \textbf{64.7} &   62.4 &   \underline{58.0 $\pm$  8.2} &   52.1 $\pm$   8.8 &   3.4 $\pm$   1.9 \\
walker-med-exp    &    6.4 &  40.1 &   81.6 &  53.8 &   57.5 &  \textbf{111.0} &   98.8 $\pm$  8.2&   98.5 $\pm$   8.5 &   0.4 $\pm$   1.5 \\
hopper-med-exp      &  111.9 &  96.3 &    0.8 &  27.1 &  110.9 &   98.7 &  \textbf{112.1 $\pm$  0.3} &   107.6 $\pm$   10.0 &   5.5 $\pm$   5.4 \\
\midrule
Mean perf          &   25.0 &  39.8 & 31.4 &  26.8 &   40.4 &   \textbf{52.3} &   \underline{48.0 $\pm$   5.4} &   \underline{47.0 $\pm$   5.7} &   7.3 $\pm$   3.0 \\
\midrule
pen-human       &   34.4 &   -1.0 &  0.6 &   12.3 &   68.9 &   37.5 &   \underline{69.6 $\pm$  22.0} &   \textbf{69.9 $\pm$ 14.2} &   0.0 $\pm$   4.0 \\
hammer-human    &    1.5 &    0.3 &  0.2 &    1.2 &    0.5 &    \textbf{4.4} &    2.5 $\pm$   1.6 &   1.7 $\pm$   1.2 &   0.2 $\pm$   0.0 \\
door-human      &    0.5 &   -0.3 & -0.3 &    0.4 &   -0.0 &    \textbf{9.9} &   -0.2 $\pm$   0.2 &   0.6 $\pm$   1.6 &   -0.3 $\pm$   0.0 \\
relocate-human  &    0.0 &   -0.3 & -0.3 &   -0.0 &   -0.1 &    \textbf{0.2} &   0.0 $\pm$   0.0 &    0.0 $\pm$   0.1 &   -0.3 $\pm$   0.0 \\
pen-cloned      &   \textbf{56.9} &   26.5 & -2.5 &   28.0 &   44.0 &   39.2 &   35.4 $\pm$  7.7 &   31.7 $\pm$   10.1 &   -3.7 $\pm$   0.5 \\
hammer-cloned   &    0.8 &    0.3 &  0.3 &    0.4 &    0.4 &    \textbf{2.1} &    0.4 $\pm$   0.0 &   0.4 $\pm$   0.1 &   0.2 $\pm$   0.0 \\
door-cloned     &   -0.1 &   -0.1 & -0.1 &    0.0 &    0.0 &    \textbf{0.4} &    0.0 $\pm$   0.0 &   0.0 $\pm$   0.0 &   -0.2 $\pm$   0.1 \\
relocate-cloned &   \textbf{-0.1} &   -0.3 & -0.3 &   -0.2 &   -0.3 &   -0.1 &   -0.2 $\pm$   0.0 &   -0.2 $\pm$   0.0 &   -0.2 $\pm$   0.0 \\
pen-expert      &   85.1 &  105.9 & -3.0 &  111.0 &  \textbf{114.9} &  107.0 &  \underline{104.8 $\pm$  22.8} &   \underline{111.8 $\pm$   12.0} &   -2.8 $\pm$   1.2 \\
hammer-expert   &  \textbf{125.6} &  127.3 & 0.3 &   39.0 &  107.2 &   86.7 &  \underline{116.6 $\pm$  9.4} &   \underline{120.0 $\pm$   5.8} &   0.2 $\pm$   0.0 \\
door-expert     &   34.9 &  103.4 &  -0.3 &  102.9 &   99.0 &  101.5 &   \textbf{104.2 $\pm$  2.0} &   \underline{104.0 $\pm$   1.4} &   -0.1 $\pm$   0.1 \\
relocate-expert &  101.3 &   98.6 & -0.4 &   91.5 &   41.6 &   95.0 &  \textbf{107.0 $\pm$   2.4} &   77.3 $\pm$   1.4 &   -0.3 $\pm$   0.0 \\
\midrule
Mean perf   &   36.7 &   38.3 &  -0.4 &   32.2 &   39.6 &   40.3 &   \textbf{45.0 $\pm$   5.7} &   \underline{43.2 $\pm$   4.3} &   -0.6 $\pm$   0.5 \\
\bottomrule
\end{tabularx}
\caption{\label{tab:performance} Evaluation of \ploff{}. We report the results of the baselines using performance results reported by \citet{fu2020d4rl}, which do not incorporate standard deviation of performances, since the numbers are based on 3 seeds. In our case we use 10 seeds, following recommendations from \citet{henderson}, and evaluate on 10 episodes for each seed before reporting average and standard deviations of performance. Results are bolded if they are best on average, underlined if within a standard deviation of the best average performance.} 
\end{table*}

Table \ref{tab:performance} shows the performance of \ploff{} on the D4RL benchmark. On average, it tops other methods on hand manipulation tasks, and tops all methods but CQL on locomotion tasks. However, even if \ploff{} performs well across the board, the approach does not solve the common failure cases shared by all methods (random datasets as well as datasets with human operated transitions, see Table \ref{tab:performance}).

\subsection{Ablations \label{sec:baselines}} We perform two ablations of the proposed method and report results in Table \ref{tab:performance}. The first one consists in using TD3 without any bonus that we note \textsc{TD3-off}. The second one is similar to \ploff{} although with a bonus based on the Euclidean distance between the concatenation of the state and the action rather than a learned distance, we refer to it as \textsc{PLOff-L2}. For \textsc{PLOff-L2}, we used the same experimental evaluation protocol and hyperparameter search as the proposed method.

The results in Table \ref{tab:performance} show that without a bonus, the performance of the policy learned is mediocre. On the other hand, \textsc{PLOff-L2} reaches performance slightly below \ploff{}. This should not come as a surprise since the Euclidean distance has been shown to be an effective measure of similarity in some continuous control tasks \citep{pwil}. Note however that in more complex tasks (typically vision-based environments), the Euclidean distance is a poor measure of similarity which makes the learning of a pseudometric necessary.

\newpage
\section{Related Work}
\label{sec:related}

\textbf{Offline Reinforcement Learning.} Offline RL \citep{lagoudakis2003least,ernst2005tree,riedmiller2005neural,pietquin2011sample,lange2012batch,levine2020offline} methods suffer from overestimation of state-action pairs that are not in the support of logged transitions. A number of methods have been explored to mitigate this phenomenon, by constraining the learned policy to be close in terms of a probabilistic distance to the behavioral policy \citep{jaques2019way,wu2019behavior, kumar2019stabilizing, siegel2020keep,peng2019advantage,pmlr-v80-fujimoto18a}, or a \textit{pessimistic} \citep{buckman2020importance} estimate of either the Q-value or the bootstrapped Q-value \citep{kumar2020conservative,laroche2019safe,simao2019safe,nadjahi2019safe}. \citet{pmlr-v80-fujimoto18a,kumar2019stabilizing,wu2019behavior} estimate the behavior policy using a conditional VAE \citep{Kingma2014AutoEncodingVB}, and constrain the policy learned offline with a distance to the behavior policy using the Kullback-Leibler divergence, the Wasserstein distance and the Maximum-Mean Discrepancy respectively.

In this work we trade the problem of estimating the behavioral policy (which is specially problematic if it is multimodal) with the computational cost of a lookup; and instead of penalizing the policy with a distribution-based distance, we learn a task relevant pseudometric instead. 
As our bonus is based on learned structural properties of the MDP we can also draw a connection to model-based off-policy RL \citep{kidambi2020morel,yu2020mopo,argenson2020model}.\\

\noindent
\textbf{State-action similarity in MDPs.} Bisimulation relations are a form of state abstraction \citep{li2006towards}, introduced in the context of MDPs by \citet{givan2003equivalence}, where states with the same rewards and transitions are aggregated. As this definition is strict, a number of works define approximate versions of it. For example, \citet{dean1997model} use a bound rather than an exact equivalence. Similarly, \citet{ferns2004metrics} introduce bisimulation metrics \citep{ferns2011bisimulation,ferns2012methods}, which use the reward signal to decide the proximity between two states. This makes bisimulation metrics close to the difference between optimal values between two states \citep{ferns2014bisimulation}. Learning a bisimulation metric online \citep{castro2020scalable,comanici2012fly} has been shown to be beneficial to learn controllable representations that eliminates unnecessary details of the state \citep{zhang2020learning} or can be used as an auxiliary loss which leads to improvement performance on the Atari benchmark \citep{gelada2019deepmdp}. \citet{castro2020scalable} introduces the Siamese network architecture, as well as a loss, to derive the bisimulation metric online. Our work differs as it learns a pseudometric on state-action pairs rather than states, and is based on offline transitions.

\citet{ravindran2003relativized} introduces MDP homomorphism as another form of abstraction which is state-action dependent rather than state dependent. Again as the partitioning induced by a homomorphism is too restrictive to be useful in practice, a number of work has looked into relaxed versions \citep{ravindran2004approximate,wolfe2006decision,van2020plannable,taylor2008bounding}.  Our work is closer to bisimulation metrics since it introduces the similarity between states-actions as a difference between reward accumulated when following the same sequence of actions (except for the first one).

\section{Concluding Remarks}

We introduced a new paradigm to compute a pseudometric-based bonus for offline RL. We learn policies consistent with the behavior policy that generated the collected transitions, and hence reduce action extrapolation error.

We showed how to derive a pseudometric from logged transitions, extending existing work from \citet{castro2020scalable} from the online to the offline setting and from pseudometrics on state space to pseudometrics on state-action space. We showed that the pseudometric we desire to learn is the fixed point of an operator, and we provide a neural architecture as well as a loss to learn it.

Conceptually, our bonus introduces a larger computational cost against other approaches that reduce extrapolation errors. We argue that this is actually a desirable direction of research for offline RL. In the presence of a fixed dataset of transitions, and since we cannot add new transitions into memory, we should insist on the other side of the well-known memory-computation trade-off.%

We demonstrated in our experimental study that our method performs comparably to existing state-of-the-art methods (tops other methods on hand manipulation task, second to top on locomotion tasks). 

\section*{Acknowledgments}
We thank Johan Ferret, Adrien Ali Ta\"iga, Saurabh Kumar and Pablo Samuel Castro for their feedback on earlier versions of the manuscript.

\bibliography{main}

\begin{thebibliography}{76}
\providecommand{\natexlab}[1]{#1}
\providecommand{\url}[1]{\texttt{#1}}
\expandafter\ifx\csname urlstyle\endcsname\relax
  \providecommand{\doi}[1]{doi: #1}\else
  \providecommand{\doi}{doi: \begingroup \urlstyle{rm}\Url}\fi

\bibitem[Argenson \& Dulac-Arnold(2021)Argenson and
  Dulac-Arnold]{argenson2020model}
Argenson, A. and Dulac-Arnold, G.
\newblock Model-based offline planning.
\newblock \emph{International Conference on Learning Representations (ICLR)},
  2021.

\bibitem[Banach(1922)]{banach1922operations}
Banach, S.
\newblock Sur les op{\'e}rations dans les ensembles abstraits et leur
  application aux {\'e}quations int{\'e}grales.
\newblock \emph{Fund. math}, 1922.

\bibitem[Bellemare et~al.(2016)Bellemare, Srinivasan, Ostrovski, Schaul,
  Saxton, and Munos]{Bellemare2016UnifyingCE}
Bellemare, M.~G., Srinivasan, S., Ostrovski, G., Schaul, T., Saxton, D., and
  Munos, R.
\newblock Unifying count-based exploration and intrinsic motivation.
\newblock In \emph{Neural Information Processing Systems (NeurIPS)}, 2016.

\bibitem[Bellman(1957)]{bellman}
Bellman, R.
\newblock A markovian decision process.
\newblock \emph{Indiana University Mathematics Journal}, 1957.

\bibitem[Bertsekas \& Tsitsiklis(1996)Bertsekas and
  Tsitsiklis]{Bertsekas1996NeuroDynamicP}
Bertsekas, D. and Tsitsiklis, J.
\newblock Neuro-dynamic programming.
\newblock 1996.

\bibitem[Bertsekas \& Tsitsiklis(1991)Bertsekas and
  Tsitsiklis]{bertsekas1991some}
Bertsekas, D.~P. and Tsitsiklis, J.~N.
\newblock Some aspects of parallel and distributed iterative algorithms-a
  survey.
\newblock \emph{Automatica}, 1991.

\bibitem[Bradbury et~al.(2018)Bradbury, Frostig, Hawkins, Johnson, Leary,
  Maclaurin, Necula, Paszke, Vander{P}las, Wanderman-{M}ilne, and Zhang]{jax}
Bradbury, J., Frostig, R., Hawkins, P., Johnson, M.~J., Leary, C., Maclaurin,
  D., Necula, G., Paszke, A., Vander{P}las, J., Wanderman-{M}ilne, S., and
  Zhang, Q.
\newblock {JAX}: composable transformations of {P}ython+{N}um{P}y programs,
  2018.
\newblock URL \url{http://github.com/google/jax}.

\bibitem[Brockman et~al.(2016)Brockman, Cheung, Pettersson, Schneider,
  Schulman, Tang, and Zaremba]{brockman2016openai}
Brockman, G., Cheung, V., Pettersson, L., Schneider, J., Schulman, J., Tang,
  J., and Zaremba, W.
\newblock Openai gym.
\newblock \emph{arXiv preprint arXiv:1606.01540}, 2016.

\bibitem[Bromley et~al.(1994)Bromley, Guyon, LeCun, S{\"a}ckinger, and
  Shah]{bromley1994signature}
Bromley, J., Guyon, I., LeCun, Y., S{\"a}ckinger, E., and Shah, R.
\newblock Signature verification using a" siamese" time delay neural network.
\newblock \emph{Neural Information Processing Systems (NeurIPS)}, 1994.

\bibitem[Buckman et~al.(2021)Buckman, Gelada, and
  Bellemare]{buckman2020importance}
Buckman, J., Gelada, C., and Bellemare, M.~G.
\newblock The importance of pessimism in fixed-dataset policy optimization.
\newblock \emph{International Conference on Learning Representations (ICLR)},
  2021.

\bibitem[Castro(2020)]{castro2020scalable}
Castro, P.~S.
\newblock Scalable methods for computing state similarity in deterministic
  markov decision processes.
\newblock In \emph{AAAI Conference on Artificial Intelligence}, 2020.

\bibitem[Comanici et~al.(2012)Comanici, Panangaden, and
  Precup]{comanici2012fly}
Comanici, G., Panangaden, P., and Precup, D.
\newblock On-the-fly algorithms for bisimulation metrics.
\newblock In \emph{International Conference on Quantitative Evaluation of
  Systems}, 2012.

\bibitem[Dadashi et~al.(2021)Dadashi, Hussenot, Geist, and Pietquin]{pwil}
Dadashi, R., Hussenot, L., Geist, M., and Pietquin, O.
\newblock Primal wasserstein imitation learning.
\newblock \emph{International Conference on Learning Representations (ICLR)},
  2021.

\bibitem[Dean \& Givan(1997)Dean and Givan]{dean1997model}
Dean, T. and Givan, R.
\newblock Model minimization in markov decision processes.
\newblock In \emph{AAAI Conference on Artificial Intelligence}, 1997.

\bibitem[Dulac{-}Arnold et~al.(2019)Dulac{-}Arnold, Mankowitz, and
  Hester]{rlchallenges}
Dulac{-}Arnold, G., Mankowitz, D.~J., and Hester, T.
\newblock Challenges of real-world reinforcement learning.
\newblock \emph{CoRR}, abs/1904.12901, 2019.

\bibitem[Ernst et~al.(2005)Ernst, Geurts, and Wehenkel]{ernst2005tree}
Ernst, D., Geurts, P., and Wehenkel, L.
\newblock Tree-based batch mode reinforcement learning.
\newblock \emph{Journal of Machine Learning Research}, 2005.

\bibitem[Ferns \& Precup(2014)Ferns and Precup]{ferns2014bisimulation}
Ferns, N. and Precup, D.
\newblock Bisimulation metrics are optimal value functions.
\newblock In \emph{Uncertainty in Artificial Intelligence (UAI)}, 2014.

\bibitem[Ferns et~al.(2004)Ferns, Panangaden, and Precup]{ferns2004metrics}
Ferns, N., Panangaden, P., and Precup, D.
\newblock Metrics for finite markov decision processes.
\newblock In \emph{Uncertainty in Artificial Intelligence (UAI)}, 2004.

\bibitem[Ferns et~al.(2006)Ferns, Castro, Precup, and
  Panangaden]{ferns2012methods}
Ferns, N., Castro, P.~S., Precup, D., and Panangaden, P.
\newblock Methods for computing state similarity in markov decision processes.
\newblock \emph{Uncertainty in Artificial Intelligence (UAI)}, 2006.

\bibitem[Ferns et~al.(2011)Ferns, Panangaden, and
  Precup]{ferns2011bisimulation}
Ferns, N., Panangaden, P., and Precup, D.
\newblock Bisimulation metrics for continuous markov decision processes.
\newblock \emph{SIAM Journal on Computing}, 2011.

\bibitem[Friedman et~al.(1977)Friedman, Bentley, and Finkel]{kd_tree}
Friedman, J., Bentley, J., and Finkel, R.
\newblock An algorithm for finding best matches in logarithmic expected time.
\newblock \emph{ACM Trans. Math. Softw.}, 1977.

\bibitem[Fu et~al.(2020)Fu, Kumar, Nachum, Tucker, and Levine]{fu2020d4rl}
Fu, J., Kumar, A., Nachum, O., Tucker, G., and Levine, S.
\newblock D4rl: Datasets for deep data-driven reinforcement learning.
\newblock \emph{arXiv preprint arXiv:2004.07219}, 2020.

\bibitem[Fujimoto et~al.(2018)Fujimoto, van Hoof, and
  Meger]{pmlr-v80-fujimoto18a}
Fujimoto, S., van Hoof, H., and Meger, D.
\newblock Addressing function approximation error in actor-critic methods.
\newblock In \emph{International Conference on Machine Learning (ICML)}, 2018.

\bibitem[Fujimoto et~al.(2019)Fujimoto, Meger, and Precup]{fujimoto2019off}
Fujimoto, S., Meger, D., and Precup, D.
\newblock Off-policy deep reinforcement learning without exploration.
\newblock In \emph{International Conference on Machine Learning (ICML)}, 2019.

\bibitem[Gelada et~al.(2019)Gelada, Kumar, Buckman, Nachum, and
  Bellemare]{gelada2019deepmdp}
Gelada, C., Kumar, S., Buckman, J., Nachum, O., and Bellemare, M.~G.
\newblock Deepmdp: Learning continuous latent space models for representation
  learning.
\newblock In \emph{International Conference on Machine Learning (ICML)}, 2019.

\bibitem[Givan et~al.(2003)Givan, Dean, and Greig]{givan2003equivalence}
Givan, R., Dean, T., and Greig, M.
\newblock Equivalence notions and model minimization in markov decision
  processes.
\newblock \emph{Artificial Intelligence}, 2003.

\bibitem[Haarnoja et~al.(2018)Haarnoja, Zhou, Abbeel, and
  Levine]{pmlr-v80-haarnoja18b}
Haarnoja, T., Zhou, A., Abbeel, P., and Levine, S.
\newblock Soft actor-critic: Off-policy maximum entropy deep reinforcement
  learning with a stochastic actor.
\newblock In \emph{International Conference on Machine Learning (ICML)}, 2018.

\bibitem[Henderson et~al.(2018)Henderson, Islam, Bachman, Pineau, Precup, and
  Meger]{henderson}
Henderson, P., Islam, R., Bachman, P., Pineau, J., Precup, D., and Meger, D.
\newblock Deep reinforcement learning that matters.
\newblock \emph{AAAI Conference on Artificial Intelligence}, 2018.

\bibitem[Jaques et~al.(2019)Jaques, Ghandeharioun, Shen, Ferguson, Lapedriza,
  Jones, Gu, and Picard]{jaques2019way}
Jaques, N., Ghandeharioun, A., Shen, J.~H., Ferguson, C., Lapedriza, A., Jones,
  N., Gu, S., and Picard, R.
\newblock Way off-policy batch deep reinforcement learning of implicit human
  preferences in dialog.
\newblock \emph{arXiv preprint arXiv:1907.00456}, 2019.

\bibitem[Kidambi et~al.(2020)Kidambi, Rajeswaran, Netrapalli, and
  Joachims]{kidambi2020morel}
Kidambi, R., Rajeswaran, A., Netrapalli, P., and Joachims, T.
\newblock Morel: Model-based offline reinforcement learning.
\newblock \emph{Neural Information Processing Systems (NeurIPS)}, 2020.

\bibitem[Kim \& Park(2018)Kim and Park]{gmmil}
Kim, K.-E. and Park, H.
\newblock Imitation learning via kernel mean embedding.
\newblock In \emph{AAAI Conference on Artificial Intelligence}, 2018.

\bibitem[Kingma \& Ba(2015)Kingma and Ba]{DBLP:journals/corr/KingmaB14}
Kingma, D.~P. and Ba, J.
\newblock Adam: A method for stochastic optimization.
\newblock In \emph{International Conference on Learning Representations
  (ICLR)}, 2015.

\bibitem[Kingma \& Welling(2014)Kingma and Welling]{Kingma2014AutoEncodingVB}
Kingma, D.~P. and Welling, M.
\newblock Auto-encoding variational bayes.
\newblock \emph{CoRR}, abs/1312.6114, 2014.

\bibitem[Konda \& Tsitsiklis(2000)Konda and Tsitsiklis]{konda}
Konda, V. and Tsitsiklis, J.
\newblock Actor-critic algorithms.
\newblock In \emph{Neural Information Processing Systems (NeurIPS)}, 2000.

\bibitem[Kumar et~al.(2019)Kumar, Fu, Tucker, and Levine]{kumar2019stabilizing}
Kumar, A., Fu, J., Tucker, G., and Levine, S.
\newblock Stabilizing off-policy q-learning via bootstrapping error reduction.
\newblock \emph{Neural Information Processing Systems (NeurIPS)}, 2019.

\bibitem[Kumar et~al.(2020)Kumar, Zhou, Tucker, and
  Levine]{kumar2020conservative}
Kumar, A., Zhou, A., Tucker, G., and Levine, S.
\newblock Conservative q-learning for offline reinforcement learning.
\newblock \emph{Neural Information Processing Systems (NeurIPS)}, 2020.

\bibitem[Lagoudakis \& Parr(2003)Lagoudakis and Parr]{lagoudakis2003least}
Lagoudakis, M.~G. and Parr, R.
\newblock Least-squares policy iteration.
\newblock \emph{Journal of Machine Learning Research}, 2003.

\bibitem[Lange et~al.()Lange, Gabel, and Riedmiller]{lange2012batch}
Lange, S., Gabel, T., and Riedmiller, M.
\newblock Batch reinforcement learning.
\newblock In \emph{Reinforcement learning}.

\bibitem[Laroche et~al.(2019)Laroche, Trichelair, and
  Des~Combes]{laroche2019safe}
Laroche, R., Trichelair, P., and Des~Combes, R.~T.
\newblock Safe policy improvement with baseline bootstrapping.
\newblock In \emph{International Conference on Machine Learning (ICML)}, 2019.

\bibitem[{Le Lan} et~al.(2021){Le Lan}, Bellemare, and
  Castro]{Lan2021MetricsAC}
{Le Lan}, C., Bellemare, M.~G., and Castro, P.~S.
\newblock Metrics and continuity in reinforcement learning.
\newblock \emph{AAAI Conference on Artificial Intelligence}, 2021.

\bibitem[Levine et~al.(2020)Levine, Kumar, Tucker, and Fu]{levine2020offline}
Levine, S., Kumar, A., Tucker, G., and Fu, J.
\newblock Offline reinforcement learning: Tutorial, review, and perspectives on
  open problems.
\newblock \emph{arXiv preprint arXiv:2005.01643}, 2020.

\bibitem[Li et~al.(2006)Li, Walsh, and Littman]{li2006towards}
Li, L., Walsh, T.~J., and Littman, M.~L.
\newblock Towards a unified theory of state abstraction for mdps.
\newblock \emph{International Symposium on Artificial Intelligence and
  Mathematics (ISAIM)}, 2006.

\bibitem[Lillicrap et~al.(2016)Lillicrap, Hunt, Pritzel, Heess, Erez, Tassa,
  Silver, and Wierstra]{lillicrap2019continuous}
Lillicrap, T.~P., Hunt, J.~J., Pritzel, A., Heess, N., Erez, T., Tassa, Y.,
  Silver, D., and Wierstra, D.
\newblock Continuous control with deep reinforcement learning.
\newblock In \emph{International Conference on Learning Representations
  (ICLR)}, 2016.

\bibitem[Lin(1992)]{experiencereplay}
Lin, L.~J.
\newblock Self-improving reactive agents based on reinforcement learning,
  planning and teaching.
\newblock \emph{Machine Learning}, 1992.

\bibitem[Mahadevan \& Maggioni(2007)Mahadevan and
  Maggioni]{protovaluefunctions}
Mahadevan, S. and Maggioni, M.
\newblock Proto-value functions: A laplacian framework for learning
  representation and control in markov decision processes.
\newblock \emph{Journal of Machine Learning Research}, 2007.

\bibitem[Melo \& Lopes(2010)Melo and Lopes]{melo}
Melo, F.~S. and Lopes, M.
\newblock Learning from demonstration using mdp induced metrics.
\newblock In \emph{European Conference on Machine Learning (ECML)}, 2010.

\bibitem[Mnih et~al.(2015)Mnih, Kavukcuoglu, Silver, Rusu, Veness, Bellemare,
  Graves, Riedmiller, Fidjeland, Ostrovski, Petersen, Beattie, Sadik,
  Antonoglou, King, Kumaran, Wierstra, Legg, and
  Hassabis]{Mnih2015HumanlevelCT}
Mnih, V., Kavukcuoglu, K., Silver, D., Rusu, A.~A., Veness, J., Bellemare,
  M.~G., Graves, A., Riedmiller, M.~A., Fidjeland, A.~K., Ostrovski, G.,
  Petersen, S., Beattie, C., Sadik, A., Antonoglou, I., King, H., Kumaran, D.,
  Wierstra, D., Legg, S., and Hassabis, D.
\newblock Human-level control through deep reinforcement learning.
\newblock \emph{Nature}, 2015.

\bibitem[Munos \& Szepesvari(2008)Munos and Szepesvari]{Munos2008FiniteTimeBF}
Munos, R. and Szepesvari, C.
\newblock Finite-time bounds for fitted value iteration.
\newblock \emph{Journal of Machine Learning Research}, 2008.

\bibitem[Nadjahi et~al.(2019)Nadjahi, Laroche, and des Combes]{nadjahi2019safe}
Nadjahi, K., Laroche, R., and des Combes, R.~T.
\newblock Safe policy improvement with soft baseline bootstrapping.
\newblock In \emph{Joint European Conference on Machine Learning and Knowledge
  Discovery in Databases}, 2019.

\bibitem[Pedregosa et~al.(2011)Pedregosa, Varoquaux, Gramfort, Michel, Thirion,
  Grisel, Blondel, Louppe, Prettenhofer, Weiss, Weiss, VanderPlas, Passos,
  Cournapeau, Brucher, Perrot, and Duchesnay]{sklearn}
Pedregosa, F., Varoquaux, G., Gramfort, A., Michel, V., Thirion, B., Grisel,
  O., Blondel, M., Louppe, G., Prettenhofer, P., Weiss, R., Weiss, R.~J.,
  VanderPlas, J., Passos, A., Cournapeau, D., Brucher, M., Perrot, M., and
  Duchesnay, E.
\newblock Scikit-learn: Machine learning in python.
\newblock \emph{Journal of Machine Learning Research}, 2011.

\bibitem[Peng et~al.(2019)Peng, Kumar, Zhang, and Levine]{peng2019advantage}
Peng, X.~B., Kumar, A., Zhang, G., and Levine, S.
\newblock Advantage-weighted regression: Simple and scalable off-policy
  reinforcement learning.
\newblock \emph{arXiv preprint arXiv:1910.00177}, 2019.

\bibitem[Pietquin et~al.(2011)Pietquin, Geist, Chandramohan, and
  Frezza-Buet]{pietquin2011sample}
Pietquin, O., Geist, M., Chandramohan, S., and Frezza-Buet, H.
\newblock Sample-efficient batch reinforcement learning for dialogue management
  optimization.
\newblock \emph{ACM Transactions on Speech and Language Processing (TSLP)},
  2011.

\bibitem[Pomerleau(1991)]{bc}
Pomerleau, D.~A.
\newblock Efficient training of artificial neural networks for autonomous
  navigation.
\newblock \emph{Neural computation}, 1991.

\bibitem[Rajeswaran et~al.(2018)Rajeswaran, Kumar, Gupta, Vezzani, Schulman,
  Todorov, and Levine]{Rajeswaran-RSS-18}
Rajeswaran, A., Kumar, V., Gupta, A., Vezzani, G., Schulman, J., Todorov, E.,
  and Levine, S.
\newblock {Learning Complex Dexterous Manipulation with Deep Reinforcement
  Learning and Demonstrations}.
\newblock In \emph{Robotics: Science and Systems (RSS)}, 2018.

\bibitem[Ravindran \& Barto(2003)Ravindran and Barto]{ravindran2003relativized}
Ravindran, B. and Barto, A.~G.
\newblock Relativized options: Choosing the right transformation.
\newblock In \emph{International Conference on Machine Learning (ICML)}, 2003.

\bibitem[Ravindran \& Barto(2004)Ravindran and Barto]{ravindran2004approximate}
Ravindran, B. and Barto, A.~G.
\newblock Approximate homomorphisms: A framework for non-exact minimization in
  markov decision processes.
\newblock 2004.

\bibitem[Riedmiller(2005)]{riedmiller2005neural}
Riedmiller, M.
\newblock Neural fitted q iteration--first experiences with a data efficient
  neural reinforcement learning method.
\newblock In \emph{European Conference on Machine Learning}, 2005.

\bibitem[Schaul et~al.(2015)Schaul, Quan, Antonoglou, and
  Silver]{schaul2015prioritized}
Schaul, T., Quan, J., Antonoglou, I., and Silver, D.
\newblock Prioritized experience replay.
\newblock \emph{International Conference on Learning Representations (ICLR)},
  2015.

\bibitem[Schmidhuber(1991)]{jurgen}
Schmidhuber, J.
\newblock A possibility for implementing curiosity and boredom in
  model-building neural controllers.
\newblock 1991.

\bibitem[Schulman et~al.(2015)Schulman, Levine, Abbeel, Jordan, and
  Moritz]{Schulman2015TrustRP}
Schulman, J., Levine, S., Abbeel, P., Jordan, M., and Moritz, P.
\newblock Trust region policy optimization.
\newblock In \emph{International Conference on Machine Learning (ICML)}, 2015.

\bibitem[Siegel et~al.(2020)Siegel, Springenberg, Berkenkamp, Abdolmaleki,
  Neunert, Lampe, Hafner, and Riedmiller]{siegel2020keep}
Siegel, N.~Y., Springenberg, J.~T., Berkenkamp, F., Abdolmaleki, A., Neunert,
  M., Lampe, T., Hafner, R., and Riedmiller, M.
\newblock Keep doing what worked: Behavioral modelling priors for offline
  reinforcement learning.
\newblock \emph{Internation Conference on Learning Representations}, 2020.

\bibitem[Silver et~al.(2014)Silver, Lever, Heess, Degris, Wierstra, and
  Riedmiller]{silver2014deterministic}
Silver, D., Lever, G., Heess, N., Degris, T., Wierstra, D., and Riedmiller, M.
\newblock Deterministic policy gradient algorithms.
\newblock In \emph{International Conference on Machine Learning (ICML)}, 2014.

\bibitem[Silver et~al.(2016)Silver, Huang, Maddison, Guez, Sifre, van~den
  Driessche, Schrittwieser, Antonoglou, Panneershelvam, Lanctot, Dieleman,
  Grewe, Nham, Kalchbrenner, Sutskever, Lillicrap, Leach, Kavukcuoglu, Graepel,
  and Hassabis]{go}
Silver, D., Huang, A., Maddison, C.~J., Guez, A., Sifre, L., van~den Driessche,
  G., Schrittwieser, J., Antonoglou, I., Panneershelvam, V., Lanctot, M.,
  Dieleman, S., Grewe, D., Nham, J., Kalchbrenner, N., Sutskever, I.,
  Lillicrap, T., Leach, M., Kavukcuoglu, K., Graepel, T., and Hassabis, D.
\newblock Mastering the game of go with deep neural networks and tree search.
\newblock \emph{Nature}, 2016.

\bibitem[Sim{\~a}o et~al.(2020)Sim{\~a}o, Laroche, and Combes]{simao2019safe}
Sim{\~a}o, T.~D., Laroche, R., and Combes, R. T.~d.
\newblock Safe policy improvement with an estimated baseline policy.
\newblock \emph{International Conference on Autonomous Agents and Multi-Agent
  Systems (AAMAS)}, 2020.

\bibitem[Sutton \& Barto(1998)Sutton and Barto]{suttonbarto}
Sutton, R. and Barto, A.
\newblock Introduction to reinforcement learning.
\newblock 1998.

\bibitem[Sutton et~al.(1999)Sutton, McAllester, Singh, and
  Mansour]{Sutton1999PolicyGM}
Sutton, R., McAllester, D.~A., Singh, S., and Mansour, Y.
\newblock Policy gradient methods for reinforcement learning with function
  approximation.
\newblock In \emph{Neural Information Processing Systems (NeurIPS)}, 1999.

\bibitem[Taylor et~al.(2008)Taylor, Precup, and Panagaden]{taylor2008bounding}
Taylor, J., Precup, D., and Panagaden, P.
\newblock Bounding performance loss in approximate mdp homomorphisms.
\newblock \emph{Neural Information Processing Systems (NeurIPS)}, 2008.

\bibitem[Thrun(1992)]{thrun}
Thrun, S.
\newblock Efficient exploration in reinforcement learning.
\newblock 1992.

\bibitem[Todorov et~al.(2012)Todorov, Erez, and Tassa]{mujoco}
Todorov, E., Erez, T., and Tassa, Y.
\newblock Mujoco: A physics engine for model-based control.
\newblock \emph{International Conference on Intelligent Robots and Systems},
  2012.

\bibitem[van~der Pol et~al.(2020)van~der Pol, Kipf, Oliehoek, and
  Welling]{van2020plannable}
van~der Pol, E., Kipf, T., Oliehoek, F.~A., and Welling, M.
\newblock Plannable approximations to mdp homomorphisms: Equivariance under
  actions.
\newblock \emph{International Conference on Autonomous Agents and Multi-Agent
  Systems (AAMAS)}, 2020.

\bibitem[Villani(2008)]{villani}
Villani, C.
\newblock Optimal transport: Old and new.
\newblock 2008.

\bibitem[Vinyals et~al.(2019)Vinyals, Babuschkin, Czarnecki, Mathieu, Dudzik,
  Chung, Choi, Powell, Ewalds, Georgiev, Oh, Horgan, Kroiss, Danihelka, Huang,
  Sifre, Cai, Agapiou, Jaderberg, Vezhnevets, Leblond, Pohlen, Dalibard,
  Budden, Sulsky, Molloy, Paine, G{\"{u}}l{\c{c}}ehre, Wang, Pfaff, Wu, Ring,
  Yogatama, W{\"{u}}nsch, McKinney, Smith, Schaul, Lillicrap, Kavukcuoglu,
  Hassabis, Apps, and Silver]{starcraft}
Vinyals, O., Babuschkin, I., Czarnecki, W.~M., Mathieu, M., Dudzik, A., Chung,
  J., Choi, D.~H., Powell, R., Ewalds, T., Georgiev, P., Oh, J., Horgan, D.,
  Kroiss, M., Danihelka, I., Huang, A., Sifre, L., Cai, T., Agapiou, J.~P.,
  Jaderberg, M., Vezhnevets, A.~S., Leblond, R., Pohlen, T., Dalibard, V.,
  Budden, D., Sulsky, Y., Molloy, J., Paine, T.~L., G{\"{u}}l{\c{c}}ehre,
  {\c{C}}., Wang, Z., Pfaff, T., Wu, Y., Ring, R., Yogatama, D., W{\"{u}}nsch,
  D., McKinney, K., Smith, O., Schaul, T., Lillicrap, T.~P., Kavukcuoglu, K.,
  Hassabis, D., Apps, C., and Silver, D.
\newblock Grandmaster level in starcraft {II} using multi-agent reinforcement
  learning.
\newblock \emph{Nature}, 2019.

\bibitem[Wolfe \& Barto(2006)Wolfe and Barto]{wolfe2006decision}
Wolfe, A.~P. and Barto, A.~G.
\newblock Decision tree methods for finding reusable mdp homomorphisms.
\newblock In \emph{The National Conference on Artificial Intelligence}, 2006.

\bibitem[Wu et~al.(2019)Wu, Tucker, and Nachum]{wu2019behavior}
Wu, Y., Tucker, G., and Nachum, O.
\newblock Behavior regularized offline reinforcement learning.
\newblock \emph{arXiv preprint arXiv:1911.11361}, 2019.

\bibitem[Yu et~al.(2020)Yu, Thomas, Yu, Ermon, Zou, Levine, Finn, and
  Ma]{yu2020mopo}
Yu, T., Thomas, G., Yu, L., Ermon, S., Zou, J., Levine, S., Finn, C., and Ma,
  T.
\newblock Mopo: Model-based offline policy optimization.
\newblock \emph{Neural Information Processing Systems (NeurIPS)}, 2020.

\bibitem[Zhang et~al.(2021)Zhang, McAllister, Calandra, Gal, and
  Levine]{zhang2020learning}
Zhang, A., McAllister, R., Calandra, R., Gal, Y., and Levine, S.
\newblock Learning invariant representations for reinforcement learning without
  reconstruction.
\newblock \emph{International Conference on Learning Representations (ICLR)},
  2021.

\end{thebibliography}
\bibliographystyle{icml2021}

\appendix
\onecolumn
{\Large \bf Appendix}

\section{Proofs \label{sec:proofs}}
\stability*
\begin{proof}
Let $d$ be a pseudometric in $\mathbb{M}$, we show that $\F(d)$ respects all properties in Definition \ref{def:pseudometric} and therefore is a pseudometric. Let $(s_1, a_1), (s_2, a_2), (s_3, a_3) \in \states \times \actions$ and their associated rewards $r_1, r_2, r_3$ and next states $s'_1, s'_2, s'_3$:
\begin{itemize}
    \item the pseudo-distance of a couple to itself is null:
\begin{flalign*}
 \F(d)(s_1, a_1; s_1, a_1) &= \underbrace{|r_1 - r_1|}_{= 0} + \gamma \E_{u \in \mathcal{U}(\actions)} \underbrace{d(s'_1, u; s'_1, u)}_{= 0 \; \text{since $d$ is a pseudometric}} = 0; &
\end{flalign*}
\item symmetry:
\begin{flalign*}
 \F(d)(s_1, a_1; s_2, a_2) &= \underbrace{|r_1 - r_2|}_{= |r_2 - r_1|} + \gamma \E_{u \in \mathcal{U}(\actions)} \underbrace{d(s'_1, u; s'_2, u)}_{= d(s'_2, u; s'_1, u) \; \text{since $d$ is a pseudometric}} = \F(d)(s_2, a_2; s_1, a_1);&
\end{flalign*}
\item triangular inequality:
\begin{flalign*}
\F(d)(s_1, a_1; s_3, a_3) &= |r_1 - r_3| + \gamma \E_{u \in \mathcal{U}(\actions)} d(s'_1, u; s'_3, u) &\\
&\leq |r_1 - r_2| + |r_2 - r_3| + \gamma \E_{u \in \mathcal{U}(\actions)} d(s'_1, u; s'_2, u) + d(s'_2, u; s'_3, u) & \\
&\leq \F(d)(s_1, a_1; s_2, a_2) + \F(d)(s_2, a_2; s_3, a_3). &
\end{flalign*}
\end{itemize}

\end{proof}

\contraction*
\begin{proof}
Let $d_1, d_2 \in \mathbb{M}$, let $(s_1, a_1), (s_2, a_2) \in \states \times \actions$ and their associated rewards $r_1, r_2$ and next states $s'_1, s'_2$, we have:

\begin{align*}
    \F(d_1)(s_1, a_1; s_2, a_2) &- \F(d_2)(s_1, a_1; s_2, a_2) \\
    &= |r_1 - r_2| - |r_1 - r_2| + \gamma \E_{u \in \mathcal{U}(\actions)} d_1(s'_1, u; s'_2, u) - \gamma \E_{u \in \mathcal{U}(\actions)} d_2(s'_1, u; s'_2, u) \\
    &= \gamma \E_{u \in \mathcal{U}(\actions)} d_1(s'_1, u; s'_2, u) - d_2(s'_1, u; s'_2, u).
\end{align*}

Therefore, we have:

\begin{flalign*}
    | \F(d_1)(s_1, a_1; s_2, a_2) - \F(d_2)(s_1, a_1; s_2, a_2) | &\leq  \gamma \E_{u \in \mathcal{U}(\actions)}  | d_1(s'_1, u; s'_2, u) - d_2(s'_1, u; s'_2, u) | & \\
    &\leq  \gamma \max_{u \in \actions}  | d_1(s'_1, u; s'_2, u) - d_2(s'_1, u; s'_2, u) | &\\
    &\leq  \gamma \max_{s, s' \in \states}\max_{u, u' \in \actions}  | d_1(s, u; s', u') - d_2(s, u; s', u') | \\
    &\leq  \gamma \|d_1 - d_2 \|_{\infty}. &
\end{flalign*}

We thus have that $\| \F(d_1) - \F(d_2) \|_{\infty} \leq  \gamma \|d_1 - d_2 \|_{\infty}$, therefore $\F$ is a $\gamma$-contraction for $\| \cdot \|_\infty$.
\end{proof}

\fixedpoint*
\begin{proof}
This is a direct application of the Banach theorem \citep{banach1922operations}. $\F$ is a $\gamma$-contracting operator with $\gamma \in [0, 1)$, in the metric space $((\states \times \actions) \times (\states \times \actions), \| \cdot \|_{\infty})$, therefore using the Banach theorem we have that $\F$ has a unique fixed point $d^*$ and $\forall d_0 \in \mathbb{M}, \lim_{n \to \infty} \F^n(d_0) = d^*$.
\end{proof}

\sampleconvergence*
\begin{proof}

The repeated application of $\hat{\F}$ is an asynchronous  fixed point iteration scheme. The convergence to $d^*$ (almost surely) is a direct application of Proposition 3 from \citet{bertsekas1991some}. Note that the state-action coverage assumption enables to apply this result since all pairs of state-action are visited an infinite number of times (almost surely).

\end{proof}

\section{Implementation \label{sec:implementation_details}}
In this section, we provide a detailed description of the experimental study. 

\paragraph{Offline datasets preprocessing.} We use datasets from \citet{fu2020d4rl}. We scale the rewards by shifting them by $ - \min_{r \sim \dataset} r$ and dividing them by $\max_{r \sim \dataset} r - \min_{r \sim \dataset} r$ for both pseudometric learning and policy learning. This enables to have comparable range of rewards between environments.

\paragraph{Pseudometric learning.} The Siamese networks $\Phi$ et $\Psi$ have the same architecture: a two-layer MLP of size $(1024, 32)$ with relu activation on top of the first layer. Note that $\Phi$ takes the concatenation of state and action as input, whereas $\Psi$ only takes the state as an input. We use a discount factor $\gamma = 0.9$ (which is different than the discount factor from the agent) in the experiments (we noticed some instabilities on human datasets, which contains less transitions than the rest of the datasets, if the discount factor is larger).

We minimize the losses $\hatLphi$ and $\hatLpsi$ using the Adam optimizer with a learning rate $10^{-3}$. The batch size used to compute the losses is $256$. The bootstrapped estimate in $\hatLphi$ is estimated with $256$ actions sampled uniformly. We train the two networks by iteratively taking a gradient step on each loss for $2.10^6$ gradient steps with parameters updated using exponential parameters averaging with a rate $\tau=0.005$

Once the $\Psi$ network is trained, we derive the $k$-nearest neighbors of each state in $\dataset$ for the distance induced by $\Psi$, with $k=50$. The nearest neighbors are computed using the scikit-learn implementation of the kd-tree algorithm, taking advantage of multiprocessing (with 50 CPUs).

\paragraph{Agent training.} We re-implemented the TD3 agent from \citet{fujimoto2019off} in \textsc{jax} \citep{jax}. We use the default hyperparameters (and did not perform HP search).

For the critic, we used a three-layer network with size $(256, 256, 1)$ with tanh activation on top of the first layer and elu activation on top of the second layer. For the policy, we used a three-layer network with size $(256, 256, |\actions|)$ where $|\actions|$ is the dimension of the action space, with tanh activation on top of the first layer, elu activation on top of the second layer and tanh activation on top of the last layer. We used the Adam optimizer for both the actor and the critic and used a learning rate of $3.10^{-4}$ (consistently with \citet{fujimoto2019off}) and trained them using batch of transitions of size 256 sampled uniformly in $\dataset$, for 500000 gradient steps.

We led experiments with the following bonuses $\bar{b}$ (and focused on the first one as it led to better empirical performance):
\begin{itemize}
\item $\bar{b}(s, a) = Q_{\bar{\omega}}(s, a) \exp(-\beta d_\dataset(s, a))$ \item $\bar{b}(s, a) = \exp(-\beta d_\dataset(s, a)) $
\item $\bar{b}(s, a) = 1 - \exp(\beta d_\dataset(s, a))$
\end{itemize}

We led a hyperparameter search for both the locomotion environments and the hand manipulation environments on $\alpha_a, \alpha_c, \beta$. We selected $\alpha_a, \alpha_c$ in $\{1, 5, 10\}$, $\beta \in \{0.1,0.25,0.5\}$ as the better combination on the average normalized performance on the tasks (averaged over 3 seeds). We re-ran the best combination of hyperparameters for 10 seeds and report results averaged over the 10 seeds and 10 evaluation episodes per seed. We found that the best combination for hand manipulation tasks was $\alpha_a = 10, \alpha_c = 10, \beta = 0.5$ and for locomotion tasks was $\alpha_a = 5, \alpha_c = 1, \beta = 0.5$.

\section{Metric Visualization \label{sec:metrics_viz}}
In this section, we show the state similarity learned by $\ploff{}$ ($\Psi$ network) and visualize it for MuJoCo locomotion environments (we could not provide such visualizations on Adroit tasks since states cannot be retrieved from observations, which is the necessary condition to generate rendering). 

\begin{figure}[h!]
\centering
\includegraphics[width=\linewidth]{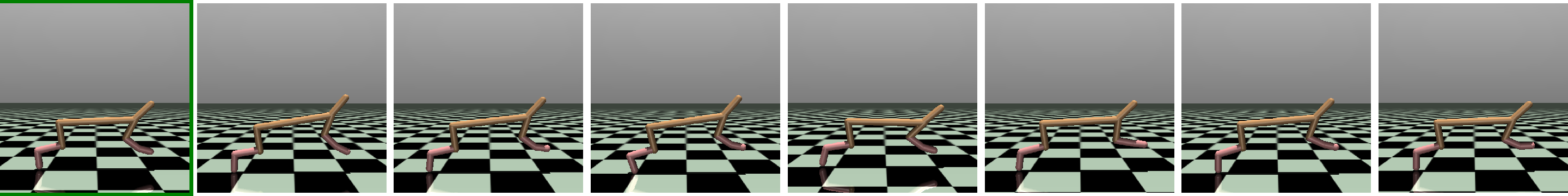}
\includegraphics[width=\linewidth]{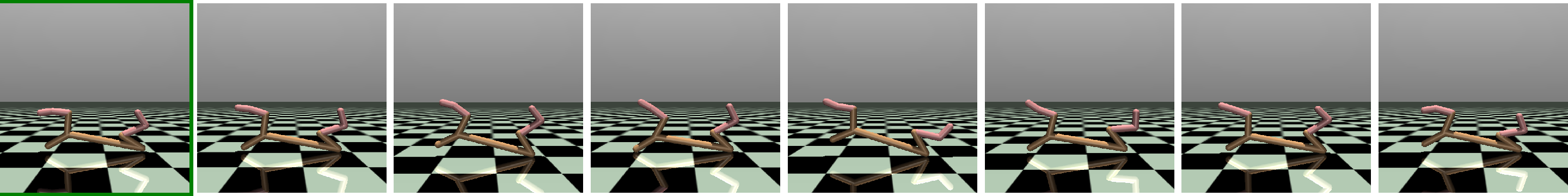}
\includegraphics[width=\linewidth]{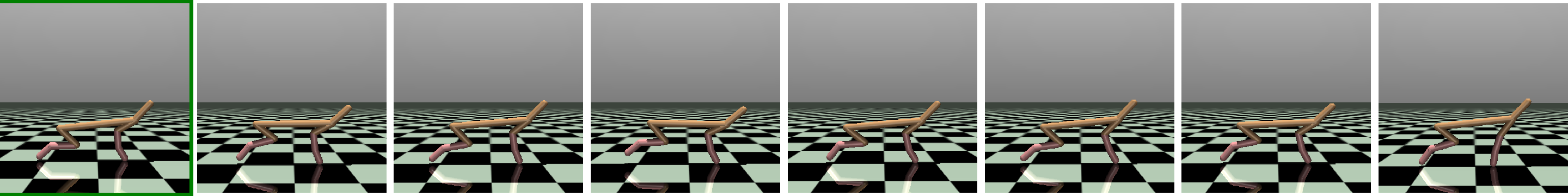}
\includegraphics[width=\linewidth]{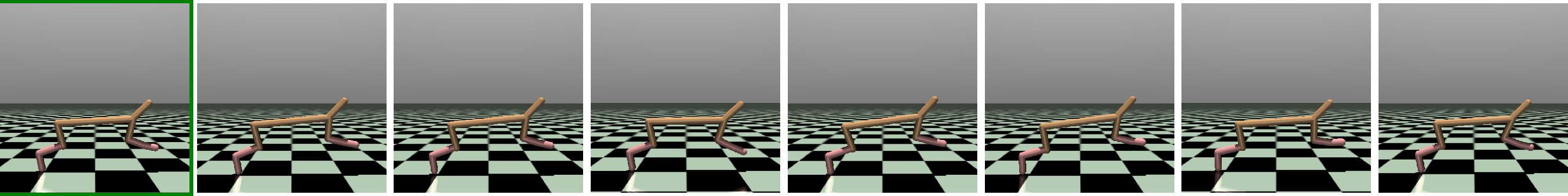}
\includegraphics[width=\linewidth]{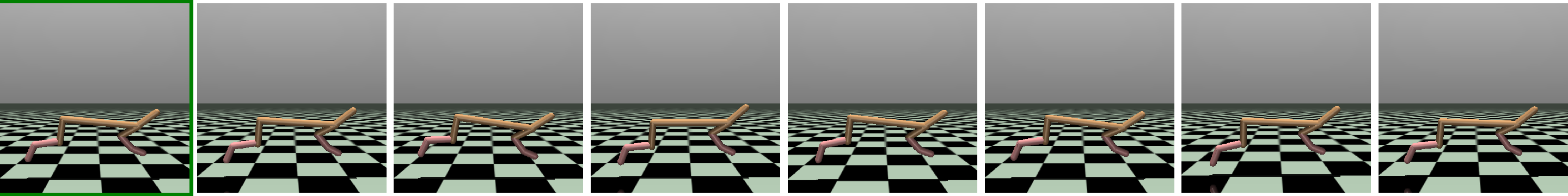}
\includegraphics[width=\linewidth]{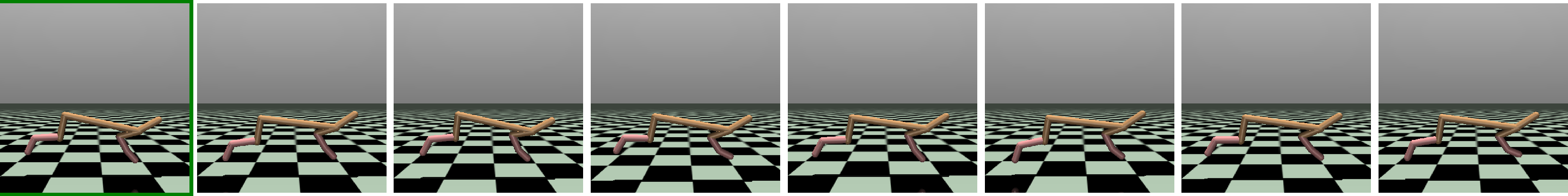}
\includegraphics[width=\linewidth]{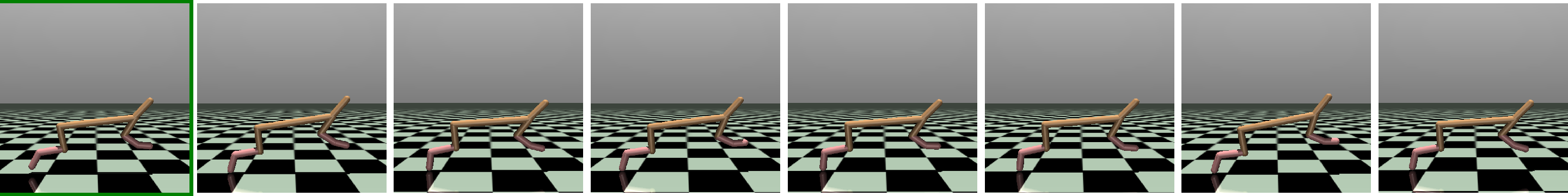}
\caption{State similarity learned by \ploff{} for HalfCheetah on the medium-replay dataset. For each row, the leftmost image is the state for which we compute nearest neighbors in the dataset $\dataset$ for the metric induced by $\Psi$ (ranked by decreasing level of similarity).}
\end{figure}

\begin{figure}
\centering
\includegraphics[width=\linewidth]{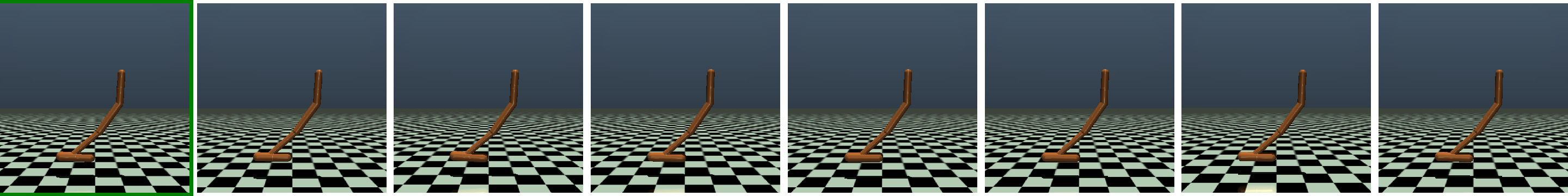}
\includegraphics[width=\linewidth]{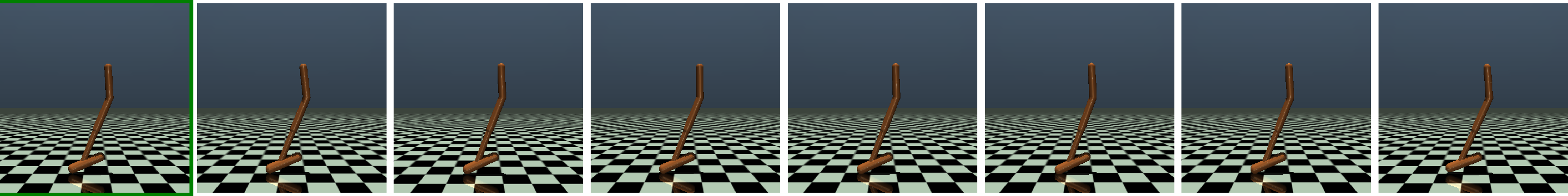}
\includegraphics[width=\linewidth]{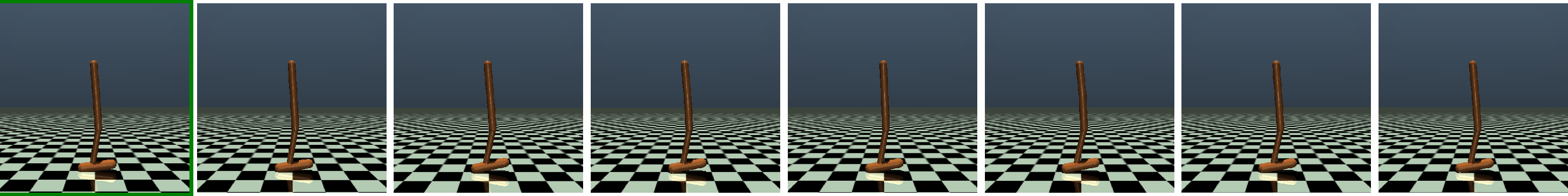}
\includegraphics[width=\linewidth]{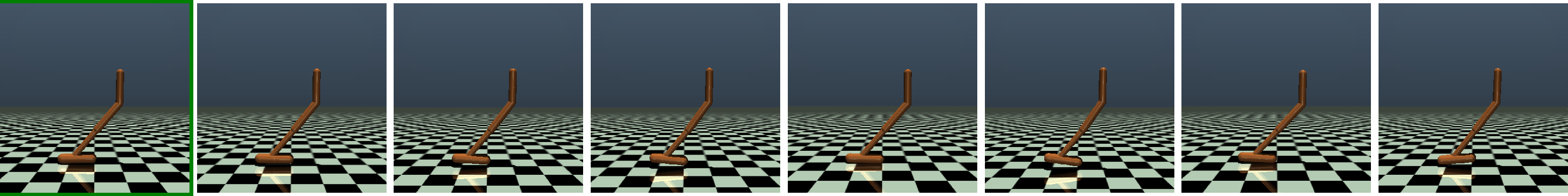}
\includegraphics[width=\linewidth]{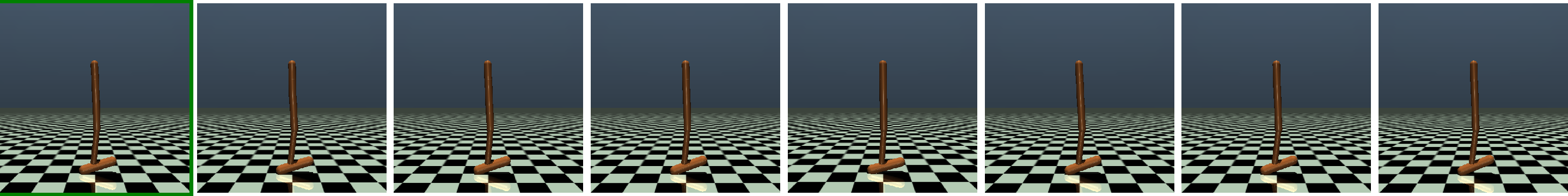}
\includegraphics[width=\linewidth]{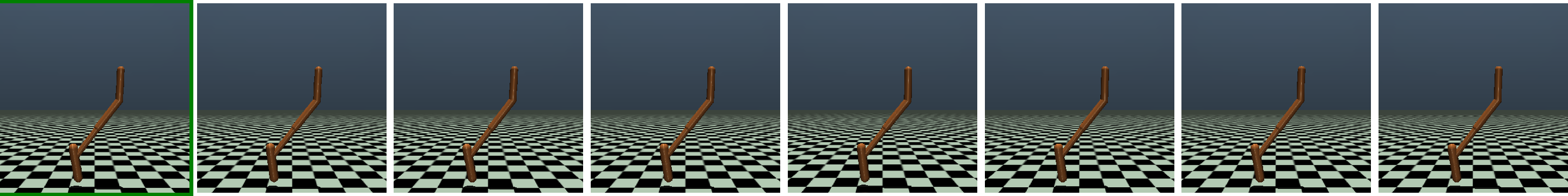}
\includegraphics[width=\linewidth]{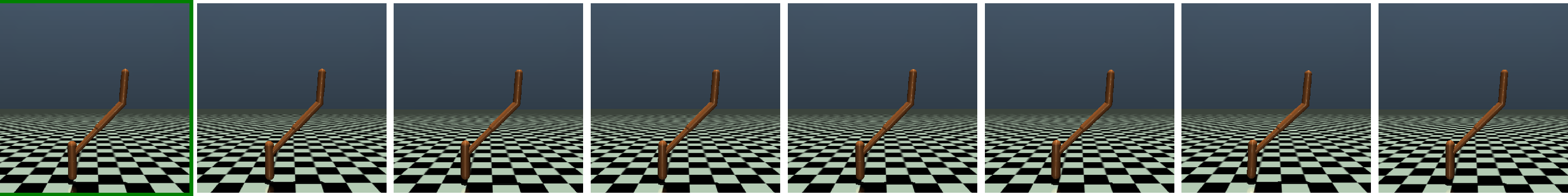}
\includegraphics[width=\linewidth]{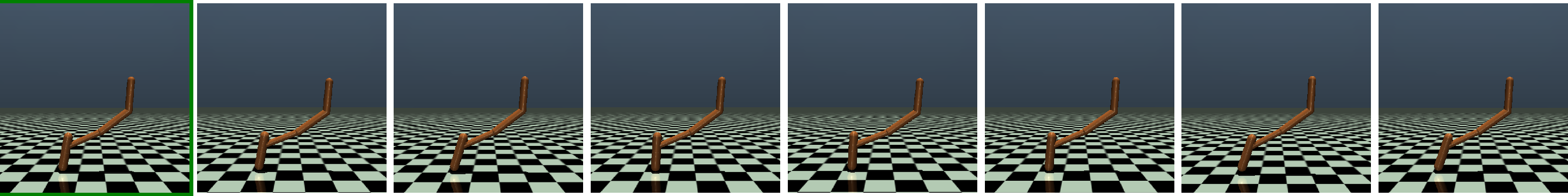}
\caption{State similarity learned by \ploff{} for Hopper on the medium-replay dataset. For each row, the leftmost image is the state for which we compute nearest neighbors in the dataset $\dataset$ for the metric induced by $\Psi$ (ranked by decreasing level of similarity).}
\end{figure}

\begin{figure}
\centering
\includegraphics[width=\linewidth]{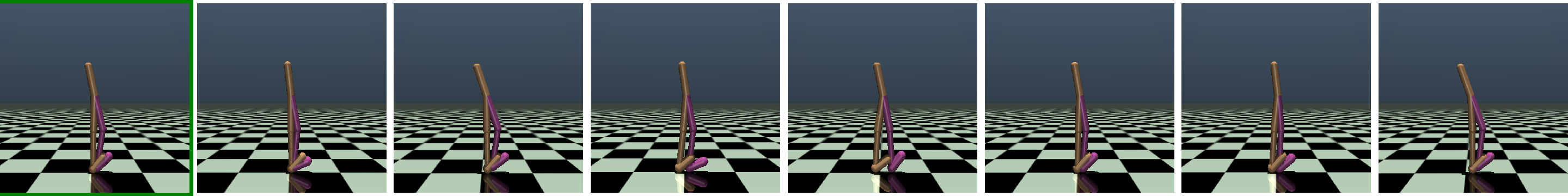}
\includegraphics[width=\linewidth]{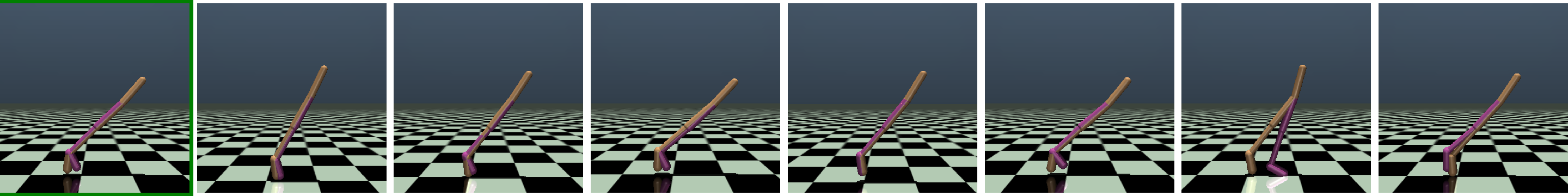}
\includegraphics[width=\linewidth]{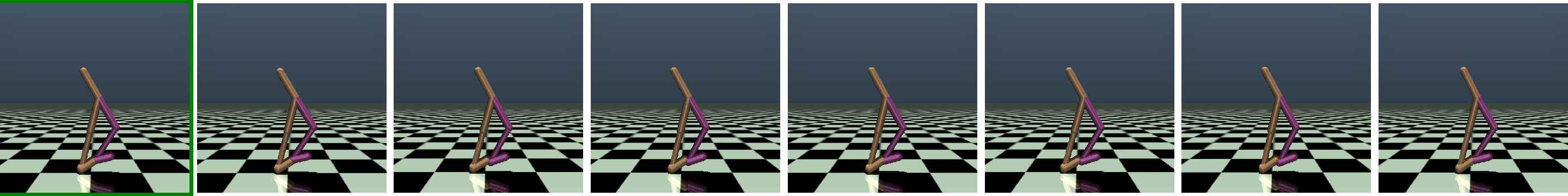}
\includegraphics[width=\linewidth]{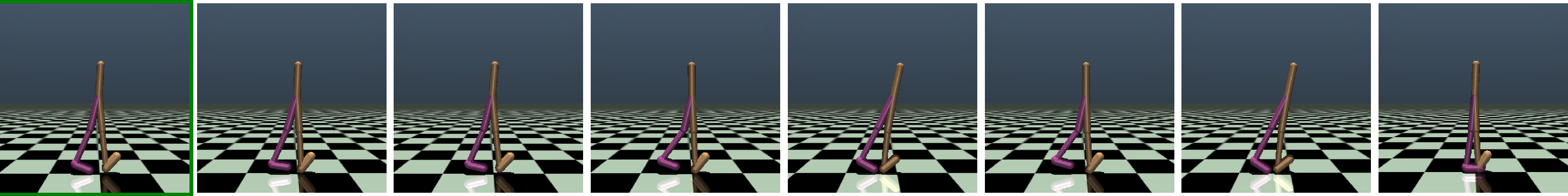}
\includegraphics[width=\linewidth]{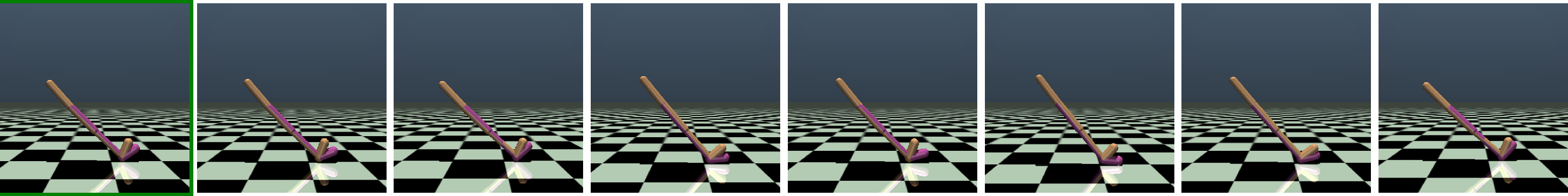}
\includegraphics[width=\linewidth]{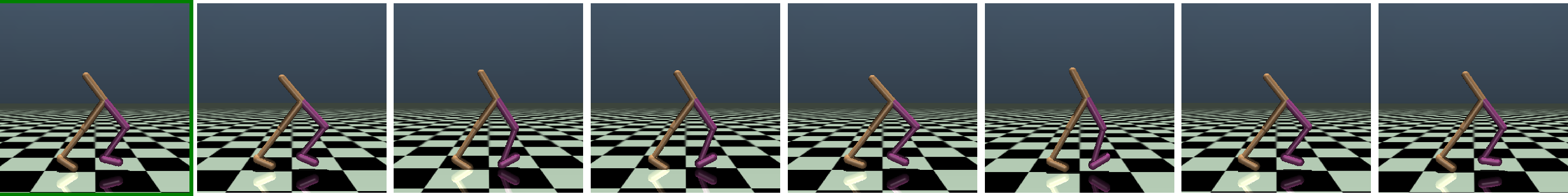}
\includegraphics[width=\linewidth]{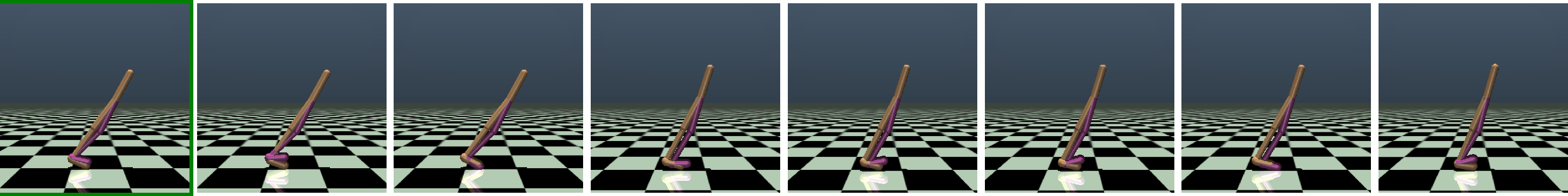}
\includegraphics[width=\linewidth]{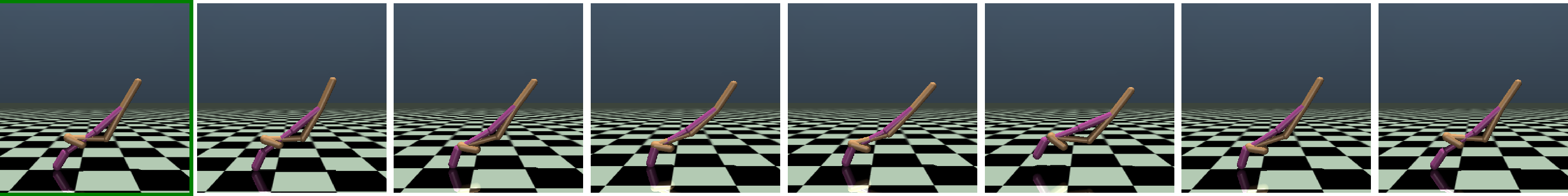}
\caption{State similarity learned by \ploff{} for Walker2d on the medium-replay dataset. For each row, the leftmost image is the state for which we compute nearest neighbors in the dataset $\dataset$ for the metric induced by $\Psi$ (ranked by decreasing level of similarity).}
\end{figure}

\end{document}